\documentclass[journal]{IEEEtran} 
\usepackage{multicol}
\usepackage{color}
\usepackage{epstopdf}
\usepackage{amssymb,amsmath,amsfonts}
\usepackage{amsthm}
\usepackage{algorithm}
\usepackage{subfigure}
\usepackage{graphicx}
\usepackage{lipsum}
\usepackage[noend]{algpseudocode}
\usepackage{color}
\usepackage{todonotes}
\usepackage{url}
\usepackage{chngcntr}
\usepackage{thmtools}
\usepackage{cite}
\newcommand{\revone}[1]{{\color{black}{#1}}}
\usepackage{mathtools}
\newcommand{\ie}[1]{\textit{i.e.,}}
\usepackage{cancel}

\newcommand{\DC}{\textsc{DiskCover}}
\newcommand{\DCT}{\textsc{DiskCoverTour}}
\newcommand{\kDCT}{\textsc{\textit{k}-DiskCoverTour}}

\newtheorem{theorem}{\textbf{Theorem}}
\newtheorem{problem}{\textbf{Problem}}

\newtheorem{lemma}{\textbf{Lemma}}

\makeatletter
\def\BState{\State\hskip-\ALG@thistlm}
\makeatother

\ifCLASSINFOpdf

\else

\fi

\begin{document}

\title{Learning a Spatial Field in Minimum Time with a Team of Robots}

\author{Varun~Suryan,~\IEEEmembership{Student Member,~IEEE,}
        and~Pratap~Tokekar,~\IEEEmembership{Member,~IEEE}
\thanks{Varun Suryan (suryan@umd.edu) and Pratap Tokekar (tokekar@umd.edu) are with the Department
of Computer Science, University of Maryland, College Park,
MD, 20740 USA. This work was done when the authors were with the Bradley Department of Electrical and Computer Engineering, Virginia Tech, Blacksburg, VA, 24060 USA.}
}
\maketitle


\begin{abstract}
We study an informative path-planning problem where the goal is to minimize the time required to learn a spatially varying entity.  \revone{We use Gaussian Process (GP) regression for learning the underlying field.} \revone{Our goal is to ensure that the GP posterior variance, which is also the mean square error between the learned and actual fields,} is below a predefined value. We study three versions of the problem. In the placement version, the objective is to minimize the number of measurement locations~\revone{while ensuring that the posterior variance is below a predefined threshold.} In the mobile robot version, we seek to minimize the total time required to visit and collect measurements from the measurement locations using a single robot. \revone{We also study a multi-robot version} where the objective is to minimize the time required by the last robot to return to a common starting location called depot. By exploiting the properties of GP regression, we present constant-factor approximation algorithms. In addition to the theoretical results, we also compare the empirical performance using a real-world dataset, with other baseline strategies.
\end{abstract}

\IEEEpeerreviewmaketitle

\section{Introduction} \label{sec:introduction}
\IEEEPARstart{S}{ensing}, modeling, and tracking various spatially varying entities can improve our knowledge and understanding of them. This can have significant economic, environmental, and health implications. For example, knowing the content of various nutrients in the soil of a farm can help the farmers better understand soil chemistry. Understanding soil chemistry is helpful for the farmers to improve the yield and reduce the application of fertilizers~\cite{aktar2009impact}. An overload of certain chemicals inside a water body may have a significant impact on marine life. Knowing the spatial extent of the spill is necessary for effective control and mitigation strategies~\cite{adzigbli2018assessing}. Understanding the spatial variation of rock minerals can help in efficient mining strategies~\cite{Blachowski2014}.
In all such applications, a key first step is the collection of data using appropriate sensors which can then be used to build models of the underlying phenomenon. However, collecting data can be tedious and often requires careful human planning. Manual data collection can also be dangerous. For example, volcano monitoring data helps to see where previous lava flows have gone and previous ash fall has occurred. However, volcanic ash is usually pulverized rocks and glass particles and potentially catastrophic for the people engaged in monitoring~\cite{1607983}. One alternative which would alleviate the human risks of manual data collection is the use of robots equipped with appropriate sensors to collect data.

There are many factors to consider when deploying robots for data collection. Usually, a trade-off must be made between the quantity of sensing resources (e.g., number of deployed robots, energy consumption, mission time) and the quality of data collected. The robots can be deployed to act as stationary or mobile sensors depending on the application~(Figure~\ref{fig:kentland}). Deploying robots to function as mobile sensors is especially challenging because of the need for path planning. While deploying mobile robotic sensors, one needs to plan the most informative resource-constrained observation paths to minimize the uncertainty in modeling and tracking \revone{the spatial phenomena}.



Planning informative resource-constrained observation paths for robot sensors to estimate a spatially varying entity, often known as Informative Path Planning~\revone{(IPP),} has received recent attention in the robotics community~\cite{krause2008near,singh2009efficient,ouyang2014multi,hollinger2014sampling,ling2016gaussian,tokekar2016sensor}. IPP deals with the problem of deciding an autonomous robot path along which maximum possible information about a quantity of interest can be extracted while operating under a set of resource constraints. In this paper, our quantity of interest is a spatially varying phenomenon, often represented by a spatial field\footnote{In this paper, a spatial field is a function, $f(x), x\in U$, that is defined over a spatial domain, $U\subset \mathbb{R}^2.$}. Generally, the underlying spatial field is specified by a probabilistic model. One of the commonly used probabilistic models is Gaussian Process (GP)~\cite{rasmussen2006gaussian}. \revone{GPs provide an mathematically convenient way of performing non-parametric regression while making fewer assumptions on the underlying field. They allow for expressing domain knowledge through the choice of kernel functions. In particular, for spatially varying fields, numerous studies have shown the efficacy of modeling with GPs~\cite{schulz2018tutorial}. An alternative would be geometric models which make strong assumptions and cannot represent the stochastic noise in the measurements directly~\cite{krause2008near}. Thus, probabilistic models make a suitable candidate for such scenarios.}
\begin{figure}
  \centering
  \includegraphics[width=0.6\linewidth]{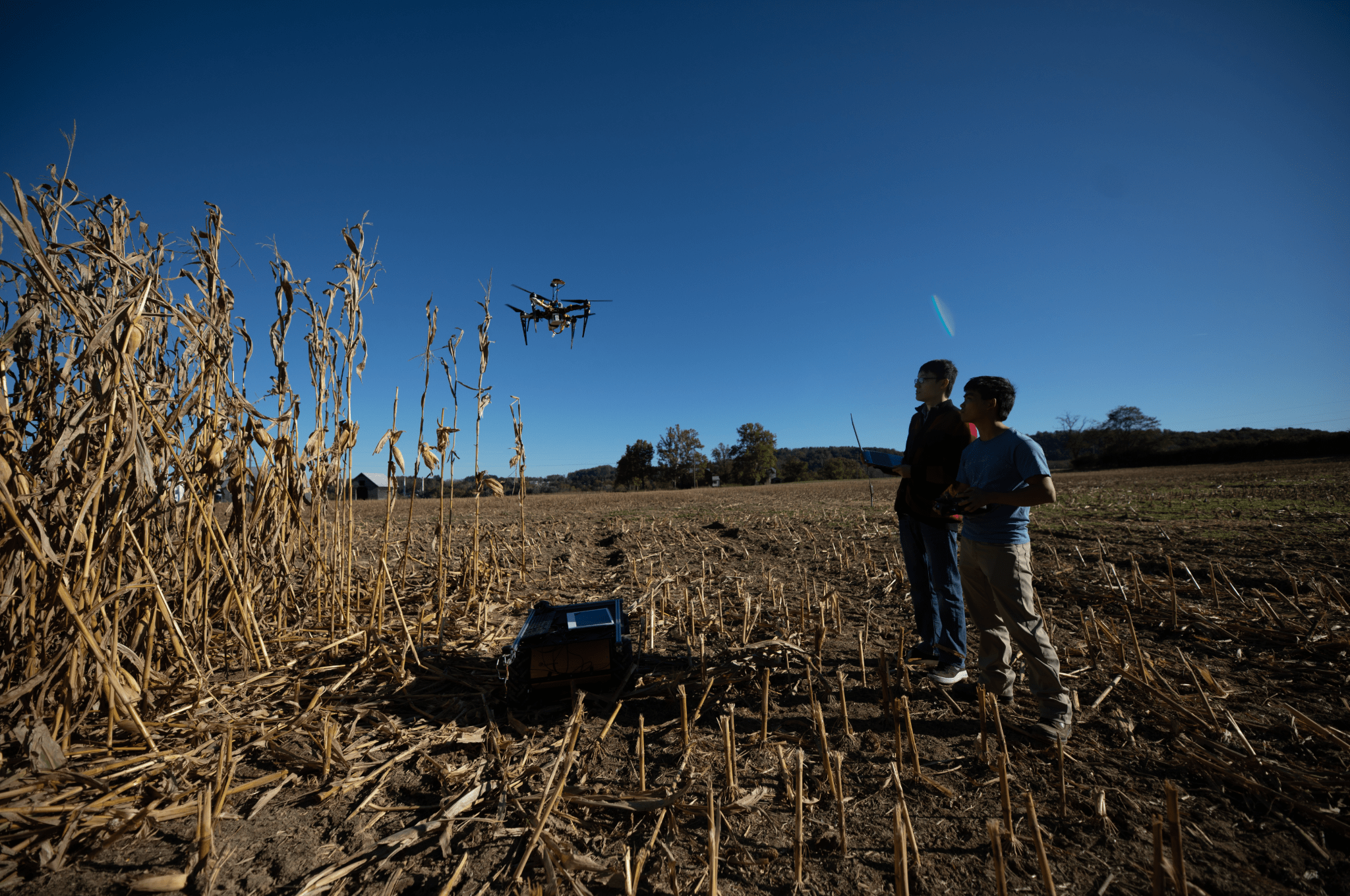}
  \caption{A single quadcopter can fly over a farm and measure the height of the crop using a LIDAR sensor.\label{fig:kentland}}
\end{figure}

\revone{Once the underlying spatial field is modeled, the next task is to plan a robot path based on that model.} The robot travels along a path \revone{planned} in this step. Several metrics can be used to perform \revone{the} planning step. An information-theoretic metric such as mutual information, entropy, or variance is typically used as a criterion to drive the robot to sampling locations~\cite{krause2008efficient}. Generally, the information-theoretic metrics are submodular and hence, an approximation guarantee can be given on the performance of the resulting algorithms~\cite{krause2008near}. Unfortunately, the information-theoretic metrics such as entropy, mutual information, etc. are indirect and do not consider the accuracy of the predictions. Unlike these works, we study how to ensure that the GP predicted mean\footnote{We use \emph{predicted} mean and \emph{estimated} mean interchangeably since the function is independent of time.} is accurate and present a constant-factor approximation algorithm if the hyperparameters of the GP kernel do not change.

We use \revone{variance of GP prediction} as the metric to perform the planning step. \revone{Predictive variance also turns out to be the Mean Square Error (MSE) in GP prediction if the hyperparameters are known~\cite{waagberg2016prediction}.} Our goal in this work is to plan the informative paths such that the predictive variance at all locations is below a predefined threshold, $\Delta$, after collecting measurements using mobile sensors. This leads to same guarantees on MSE as well.
We study three related problems, that of:
\revone{
\begin{enumerate}
    \item finding measurement locations to make measurements;
    \item planning a tour for a single robot to visit those measurement locations; and
    \item planning tours for multiple mobile robots
\end{enumerate}}
to ensure that the \revone{predictive variance} is below $\Delta$. The objective is to minimize the number of measurement locations in the first problem and the total tour time in the second problem. \revone{With} multiple robots, the objective is to minimize the maximum time taken among all the robots. The total tour time is given by the measurement time and the travel time between measurement locations. The measurement time depends on the number of measurements taken at each location as well as the time to take a single measurement. Depending on the sensor, the measurement time can be zero (e.g., cameras) or non-zero (e.g., soil probes measuring organic content). We show that a non-adaptive algorithm suffices to solve the problem and yields a polynomial-time constant-factor approximation to the optimal algorithm. While other algorithms have been proposed \revone{before} for estimating spatial fields, this is the first result that provides the theoretical guarantees on the total time for ensuring predictive accuracy at all points. Our main contributions include:
\begin{itemize}
    \item introducing stationary sensor placement and mobile sensor algorithms for ensuring that the \revone{predictive variance, and hence MSE,} at each location in the environment is below a predefined threshold,
    \item providing polynomial-time constant-factor approximation guarantees on their performance, and
    \item showing their performance on a real-world dataset comprising of organic matter concentrations at various locations within a farm.
\end{itemize}
Similar problems have been studied in the literature. \revone{For example, Yfantis et. al.~\cite{yfantis1987efficiency} studied a stationary sensor problem. Their approach considers and investigates only three types of pre-defined placement designs while for a general case none of them may be a good design. The algorithms presented in this work are not restricted to any pre-defined placement strategy.} Further, we are not aware of any existing theoretical guarantees on the mobile sensor problems studied in this paper.

The rest of the paper is organized as follows: In section~\ref{sec:lit_review}, we present a discussion on related works and background on the problems studied. In section~\ref{ch:algorithms}, we formally present the problems and their solutions. Simulation results are presented in the section~\ref{ch:results} followed by the conclusion and scope for the future work in section~\ref{ch:conclusion}.

A preliminary version of this work was presented at the $13^{th}$ International Workshop on the Algorithmic Foundations of Robotics (WAFR'18)~\cite{suryan2018sensor}. In the preliminary version, we provided guarantees for the chance constraints of incorrect predictions~\revone{using an aggregate measure of prediction error}. In current work, a more \revone{direct} performance criterion,~\revone{MSE at each location of the environment}, is considered which leads to stronger theoretical guarantees. Also, an extension of the algorithms for the multi-robot case is presented.

\section{Related Work and Background} \label{sec:lit_review}
We begin by reviewing the related work in sensor placement where the goal is to cover a given environment using sensors placed at fixed locations and mobile sensing where sensors can move and collect measurements from different locations.
\subsection{Stationary Sensor Placement}
When monitoring a spatial phenomenon, such as temperature or humidity in an environment, selection of a limited number of sensors and their locations is an important problem. The goal in this problem is to select the best $k$ out of $n$ possible sensor locations and use the measurements from these to make predictions about the spatial phenomenon. The typical formulation of a sensor selection problem makes it NP-hard~\cite{bian2006utility}. Previous work used global optimization techniques such as branch and bound to exactly solve this problem~\cite{welch1982branch, kirkpatrick1983optimization}. However, these exact approaches are often computationally intensive.

One can solve the task as an instance of the art-gallery problem~\cite{hochbaum1985approximation, Gonzalez-Banos:2001:RAA:378583.378674}: Find the minimum set of guards inside a polygonal workspace from which the entire workspace is visible. However, this version of the problem only covers vision-based sensors and does not consider noisy measurements~\cite{krause2008near}.

An alternative approach from spatial statistics is to learn a model of the phenomenon, typically as a GP~\cite{caselton1984optimal, cressie1992statistics}. The learned GP model can then be used to predict the effect of placing sensors at locations and thus optimize their placement. For a given GP model, many criteria including information-theoretic ones have been proposed to evaluate the quality of placement. Shewry and Wynn introduced the maximum entropy criterion~\cite{shewry1987maximum} where the sensors are placed sequentially at the locations of highest entropy. Ko et al.~\cite{10.2307/171694} proposed a greedy algorithm by formulating the entropy maximization as maximizing the determinant of the covariance matrix. However, the entropy criterion tends to place the sensors at the boundary of the environment thus wasting sensed information~\cite{ramakrishnan2005gaussian}. Mutual information (MI) can be used as well~\cite{caselton1984optimal, zimmerman2006optimal, Zhu2006}. Krause et al.~\cite{krause2008near} study the problem of maximizing MI for optimizing sensor placement problem. They present a polynomial-time approximation algorithm with constant factor guarantee by exploiting  submodularity~\cite{nemhauser1978analysis}. Eventually, they show that MI criterion leads to improved accuracy with a fewer number of sensors compared to other common design criteria such as entropy~\cite{shewry1987maximum}, A-optimal, D-optimal, and E-optimal design~\cite{atkinson2014optimal}.

The above-mentioned methods estimate the prediction error indirectly. Nguyen et al.~\cite{6466575} consider choosing a set of $n$ potential sensor measurements such that the root mean square prediction error is minimized. They present an annealing based algorithm for the sensor selection problem. Their algorithm starts by selecting a potential subset of cardinality $k$ from the entire population of sensor locations. After that, it iteratively attempts to substitute the members of the selected subset by its neighbors according to an optimization criterion.

None of the criteria discussed above cannot directly make any guarantees on the MSE in predictions at each point in the environment. Instead, we design a sensor placement algorithm which results in an accurate reconstruction of the spatial field using the collected sensor measurements. Most works in the past have focused on optimizing an objective function (entropy, MI, etc.) given the resource constraints (limited energy, number of sensors, and time, etc.). We optimize the resource requirement given the objective constraint (MSE below a predefined threshold $\Delta$), predictive accuracy more than a predefined threshold in our case.

\subsection{Mobile Sensing}
The goal in the mobile sensing problem, also known as Informative Path Planning (IPP), is to compute paths for robots acting as mobile sensors. Paths are being computed in order to accurately estimate some underlying phenomenon, typically a spatial field~\cite{jawaid2015informative, krause2007nonmyopic}. A central problem in IPP is to identify the hotspots in a large-scale spatial field. Hotspots are the regions in which the spatial field measurements exceed a predefined threshold. In many applications, it is necessary to assess the spatial extent and shape of the hotspot regions accurately. Low et al. presented a decentralized active robotic exploration strategy for probabilistic classification/labeling of hotspots in a GP-based spatial field~\cite{low2012decentralized}. The time needed by their strategy is independent of the map resolution and the number of robots, thus making it practical for \textit{in situ}, real-time active sampling. Another formulation in hotspot identification is that of level set identification~\cite{galland2004synthetic}. 

Previous works on level set boundary estimation and tracking~\cite{singh2006active, dantu2007detecting, srinivasan2008ace} have primarily focused on communication of the sensor nodes, without giving much attention to individual sampling locations. Bryan et al.~\cite{bryan2006active} proposed the straddle heuristic, which selects sampling locations by trading off uncertainty and proximity to the desired threshold level, both estimated using GPs. However, no theoretical justification had been given for its use and its extension to composite functions~\cite{bryan2008actively}. Gotovos et al. proposed a level set estimation algorithm, which utilizes GPs to model the target function and exploits its inferred confidence bounds to drive the selection process. They provided an information-theoretic bound on the number of measurements needed to achieve a certain accuracy, when the underlying function is sampled from a GP~\cite{gotovos2013active}.

In many mobile sensing problems, it is not enough to identify only a few specific regions but estimate the entire spatial field accurately. It can be formulated as a path planning problem to observe a spatial field at a set of sampling locations, and then making inference about the unobserved locations~\cite{861310}. Choosing and visiting the sample locations so that one can have an accurate prediction (point prediction and/or prediction interval) is of great importance in soil science, agriculture, and air pollution monitoring~\cite{Zhu2006}. The objective functions used are usually submodular and thus exhibit a diminishing returns property. Submodularity arises since \revone{nearby} measurement locations are correlated~\cite{krause2011submodularity}. Chekuri and Pal introduced a quasi-polynomial time algorithm~\cite{chekuri2005recursive} for maximizing a submodular objective along the path using a~\textit{recursive greedy strategy}. This algorithm was further extended by Binney et al.~\cite{binney2013optimizing} for spatiotemporal fields using average variance reduction~\cite{das2008algorithms} as the objective function.

Zhang and Sukhatme proposed an adaptive sampling algorithm consisting of a set of static nodes and a mobile robot tasked to reconstruct a scalar field~\cite{zhang2007adaptive}. They assume that the mobile robot can communicate with all the static nodes and acquire sensor readings from them. Based on this information, a path planner generates a path such that the resulting integrated mean square error is minimized subject to the constraint that the boat has a finite amount of energy.

An important issue in designing robot paths is deciding the next measurement location~\cite{leonard2007collective, popa2006ekf, sim2005global, singh2009efficient}, often referred to as the exploration strategy. Traditionally, conventional sampling methods~\cite{rahimi2004adaptive} such as raster scanning, simple random sampling, and stratified random sampling have been used for single-robot exploration. Low et al. presented an adaptive exploration strategy called adaptive cluster sampling. It was demonstrated to reduce mission time and yield more information about the environment~\cite{low2007adaptive}. Their strategy performs better than a baseline sampling scheme called systematic sampling~\cite{tuchscherer1993thompson} using root mean squared error as a metric. A different adaptive multi-robot exploration strategy called MASP  was presented in~\cite{low2008adaptive} which performs both wide-area coverage and hotspot sampling using non-myopic path planning. MASP allows for varying adaptivity and its performance is theoretically analyzed. Further, it was demonstrated to sample efficiently from a GP and \textit{log}GP. However, the time complexity of implementing it depends on the map resolution, which limits its large-scale use. To alleviate this computational difficulty, an information-theoretic approach was presented in~\cite{low2009information}. The time complexity of the new approach was independent of the map resolution and less sensitive to the increasing robot team size. Garnett et al.~\cite{garnett2012bayesian} considered the problem of active search, which is also about sequential sampling from a domain of two (or more) classes. Their goal was to sample as many points as possible from one of the classes.


Yilmaz et al.~\cite{yilmaz2008path} solved the adaptive sampling problem using mixed integer linear programming. Popa et al.~\cite{popa2006ekf} posed the adaptive sampling problem as a sensor fusion problem within the extended Kalman filter framework. Hollinger and Sukhatme proposed a sampling-based motion planning algorithm that generates maximal informative trajectories for the mobile robots to observe their environment~\cite{hollinger2014sampling}. Their information gathering algorithm extends ideas from rapidly-exploring random graphs. Using branch and bound techniques, they achieve efficient optimization of information gathering while also allowing for operation in continuous space with motion constraints. Low et al.~\cite{cao2013multi} presented two approaches to solve IPP for \textit{in situ} active sensing of GP-based anisotropic spatial fields. Their proposed algorithms can trade-off active sensing performance with computational efficiency. Ling et al.~\cite{ling2016gaussian} proposed a nonmyopic adaptive GP planning framework endowed with a general class of Lipschitz continuous reward functions. Their framework can unify some active learning/sensing and Bayesian optimization criteria and offer practitioners flexibility to specify choices for defining new tasks. Tan et al.~\cite{inproceedingsKobilarov} introduced the receding-horizon cross-entropy trajectory optimization. Their focus was to sample around regions that exhibit extreme sensory measurements and much higher spatial variability, denoted as the region of interest. They used GP-UCB~\cite{srinivas2009gaussian} as the optimization criteria which helps in exploring initially and converging on regions of interest eventually.

A naive implementation of GP prediction scales poorly with increasing training dataset size. Sparse GP frameworks can overcome this problem by using only a subset of the data to provide accurate estimates. A state-of-the-art sparse GP variant is SPGP~\cite{snelson2006sparse, csato2002sparse, Seeger:161318, Hoang:2015:UFA:3045118.3045180}. The SPGP framework learns a pseudo subset that best summarizes the training data. Mishra et al. introduced an online IPP framework AdaPP~\cite{mishra2018online} which uses SPGP. 


\subsection{Sensing with Multiple Robots}
Mobile sensing can be made faster by distributing the task among several robots. Multi-robot systems can do complex tasks and have been widely used in environmental sampling~\cite{luo2016distributed}, coverage~\cite{cortes2004coverage}. Robots can use local communication or control laws to achieve some collective goals.

Singh et al.~\cite{singh2009efficient} proposed a sequential allocation strategy that uses GP regression, which can be used to extend any single robot planning algorithm for the multi-robot problem. Their procedure approximately generalizes any guarantees for the single-robot problem to the multi-robot case. However, the approach works only when MI is the optimization objective. Cao et al.~\cite{cao2013multi} presented two approaches along with their complexity analysis addressing a trade-off between active sensing performance and time efficiency. Luo et al.~\cite{luo-and-sycara-2018-107469} combined adaptive sampling with information-theoretic criterion into the coverage control framework for model learning and simultaneous locational optimization. They presented an algorithm allowing for collaboratively learning the generalized model of density function using a mixture of GPs with hyperparameters learned locally from each robot. Kemna et al.~\cite{kemna2017multi} created a decentralized coordination approach which first splits the environments into Voronoi partitions and makes each vehicle then run within their own partition. Other multi-robot approaches used in other domains, e.g. exploration and estimation with ground vehicles, include auction-based methods~\cite{zlot2002multi,simmons2000coordination, sheng2004multi} and spatial segregation, typically through Voronoi partitioning~\cite{soltero2012generating, marino2015decentralized}.

Tokekar et al.~\cite{tokekar2016sensor} presented a constant factor approximation algorithm for the case of accurately classifying each point in a spatial field. The first step in the algorithm is to determine potentially misclassified points and then to find a tour visiting neighborhoods of each potentially misclassified point. In this paper, we study a regression version of the problem where every point is of interest. We exploit the properties of GP and squared-exponential kernel to find a constant-factor approximation algorithm. Before the details of the algorithms, we review some relevant background and useful properties of GPs and MSE.
\subsection{Gaussian Processes}\label{sec:kernel}
In GP regression, the posterior variance at any test location $x$ is given by,
\begin{equation}\label{posterior_var}
\hat{\sigma}^2_{x|X} = k(x,x) - \mathbf{k}(x, X)\left[\mathbf{K}(X,X)+\omega^2\mathbf{I}\right]^{-1} \mathbf{k}(X, x), 
\end{equation}
\revone{where $\mathbf{K}(X, X)$ is the kernel matrix with entries, $\mathbf{K}_{pq}=k(x_p, x_q)=\sigma_0^2\exp\left(\frac{||x_p-x_q||^2}{2l^2}\right)$. Here, $\sigma_0^2,~l,~\omega^2$ are signal variance, length scale, and additive independent and identically distributed Gaussian measurement noise respectively~\cite{rasmussen2006gaussian}. We use the same value of length scale along each input dimension.} Note that the posterior variance at a particular location $x$ conditioned on set of observations at locations $X = \{x_1, \ldots , x_n\}$ does not depend on the actual observation but only on the locations from where the observations are collected. Multiple observations at a location is equivalent to that location being counted as many times as the number of measurements. The kernel is a function that measures the \emph{similarity} between two measurement locations~\cite{rasmussen2006gaussian}.

Since the posterior variance is a function of only the measurement locations, the posterior variance for all points in the environment can be computed \emph{a priori}, if the measurement locations are known, even without making any observations.
In many implementations~\cite{krause2008efficient,krause2007nonmyopic,meliou2007nonmyopic}, the hyperparameters for the kernel $k$ are tuned online as more data is gathered. As such, the hyperparameters may change with the observed data and the posterior variance will depend on the data observed, which may require adaptive planning. We assume that the hyperparameters are estimated \emph{a priori}. \revone{This is done using prior data from the same or similar environments or a pilot deployment over a smaller region, as described in~\cite{hollinger2012uncertainty,o2012gaussian,krause2008near}. Example applications are underwater inspection~\cite{hollinger2012uncertainty} and occupancy map building~\cite{o2012gaussian}, where prior data is used for determining hyperparameters before the actual deployment.} Nevertheless, one can perform sensitivity analysis of the presented algorithms by varying the hyperparameters~\cite{kennedy2001bayesian,oakley2004probabilistic}.

The posterior mean $\hat{\mu}_{x|X}$ at a location $x$ is given by a weighted linear combination of the observed data,
\begin{equation}\label{mean}
\hat{\mu}_{x|X} = \mathbf{k}(x, X)\left[\mathbf{K}(X,X)+\omega^2\mathbf{I}\right]^{-1}\mathbf{y}, 
\end{equation}
where $\mathbf{y} = \{y_1, \ldots, y_n\}$ denotes the observations at locations $X = \{x_1, \ldots, x_n\}$. 
\subsection{Mean Square Error}
MSE measures the expected squared difference between an estimator and the parameter the estimator is designed to estimate~\cite{hastie2005elements}. The MSE at a location $x$ for an estimator $\hat{f}$ is,
\begin{equation}\label{eq:bias_variance}
    MSE\left(\hat{f}(x)\right) = Var\left(\hat{f}(x)\right) + \left(E\revone{[}\hat{f}(x)-f(x)\revone{]}\right)^2,
\end{equation}
where the Equation~\ref{eq:bias_variance} is the bias-variance expression for the estimator. GP predicted value $\hat{f}(x)$ at a location $x$ is an unbiased estimator of the true value $f(x)$~\cite{yfantis1987efficiency} and has a normal distribution with mean given by Equation~\ref{mean}, and variance given by Equation~\ref{posterior_var}. \revone{Among all linear and non-linear estimators, GP is the best in terms of minimizing MSE~\cite{waagberg2016prediction, stein2012interpolation}.} Further, GPs are unbiased and hence, the MSE at a location $x$ is equal to the posterior variance of the predicted value, \ie,
\begin{equation}\label{eq:mse_postVar}
    MSE(x) = \hat{\sigma}^2_{x|X}.
\end{equation}
From Equation~\ref{eq:mse_postVar}, one can deduce that MSE for GPs is same as the posterior variance and hence, any guarantees for the posterior variance hold for MSE as well.

\section{Algorithms} \label{ch:algorithms}
In this section, we formally define the problems and the algorithms. We assume that the environment is a two dimensional area $U \subset \mathbb{R}^2$ and the underlying spatial field is an instance of a GP, $F$~\cite{yfantis1987efficiency}. $F$ has an isotropic covariance function of the form,
\begin{equation}\label{eq:kerenldef}
    C_Z(x, x') = \sigma_0^2\exp\left(-\frac{(x-x')^2}{2l^2}\right); \forall x, x'\in U,
\end{equation}
defined by a squared-exponential kernel where the hyperparameters $\sigma_0^2$ and $l$ are known \emph{a priori}.
Let $X$ denote the set of measurement locations within $U$ produced by an algorithm.
\begin{problem}[Placement]
Find the minimum number of measurement locations, such that the MSE at each location in $U$ is below $\Delta \revone{< \sigma_0^2}$, \ie,
\begin{equation*}
\begin{aligned}
& \underset{}{\text{minimize}}
& & |X|, \\
& \text{subject to}
& & MSE(x) \leq \Delta, \forall x \in U,
\end{aligned}
\end{equation*}
where $|X|$ is the cardinality of $X$ and $MSE(x)$ is the MSE at location $x$. \label{prob:placement}
\end{problem}

\begin{problem}[Mobile]\label{prob:mobile}
Find the minimum time trajectory for a mobile robot that obtains a finite set of measurements at one or more locations in $U$, such that the MSE at each location in $U$ \revone{is} less than $\Delta$, \ie, 
\begin{equation*}
\begin{aligned}
& \underset{}{\text{minimize}}
& & len(\tau) + \eta n(X),  \\
& \text{subject to}
& & MSE(x) \leq \Delta, \forall x \in U.
\end{aligned}
\end{equation*}
$\tau$ denotes the tour of the robot. Robot travels at unit speed, obtains one measurement in $\eta$ units of time and obtains $n(X)$ total measurements. 
\end{problem}
The robot may be required to obtain multiple measurements from a single location. Therefore, the number of measurements $n(X)$ can be more than $|X|$. \revone{For multiple robots}, their tours can start at the same starting location (often referred to as a \emph{depot}) or can start at different locations. In this paper, our focus is on the former case. The latter case is more appropriate when the robots must persistently monitor the environment.

\begin{problem}[multi-robot]
For $k$ robots starting from a given starting location (depot), design a set of trajectories that collectively obtain a finite set of measurements at one or more locations in $U$, such that the MSE at each location in $U$ is less than $\Delta$, \ie,
\begin{equation*}
\begin{aligned}
& \underset{}{\text{minimize}}
& & \underset{i\in\lbrace1,\ldots,k\rbrace}{\text{max}} len(\tau_i)  + \eta n(X_i),\\
& \text{subject to}
& & MSE(x)\leq\Delta, \forall x \in U.
\end{aligned}
\end{equation*}
$\tau_i$ denotes the tour of the $i^{th}$ robot and $X_i$ the subset of measurement locations covered by the $i^{th}$ robot. The robots travel at unit speed, obtain one measurement in $\eta$ units of time. $i^{th}$ robot obtains $n(X_i)$ total measurements.
\label{prob:multi-robot}
\end{problem}
The solution for Problem~\ref{prob:placement} is a subset of the solution to Problem~\ref{prob:mobile}. Further, the solution for Problem~\ref{prob:multi-robot} is derived from the solution for Problem~\ref{prob:mobile}.
The three algorithms build on top of each other by: $(1)$ finding a finite number of measurement locations for the robot; $(2)$ finding a tour to visit all the measurement locations; and $(3)$ splitting the tour from step $2$ in multiple sub-tours for $k$ robots. We exploit the properties of squared-exponential kernel to find the measurement locations. By knowing the value at a certain point within some tolerance, values at nearby points can be predicted albeit up to a larger tolerance. 

\subsection{Necessary and Sufficient Conditions}
We start by deriving necessary conditions on how far a test location can be from its nearest measurement location. A test location corresponds to a point in the environment where we would like to make a prediction.
\begin{lemma}[Necessary Condition] \label{lemma_var_tol_exp}
For any test location $x$, if the nearest measurement location is at a distance $r_{max}$ away, and,
\begin{equation}
r_{max} > l\sqrt[]{-\log\left(1- \frac{\Delta}{\sigma_0^2}\right)},
\end{equation}
then it is not possible to bring down the MSE below $\Delta$ at $x$.
\end{lemma}
\begin{proof}
Consider the posterior variance $\hat{\sigma}_x^2(n)$ at $x$ (which is also equal to the $MSE$ at $x$ from Equation~\ref{eq:mse_postVar}) after collecting $n$ measurements, possibly from different locations. A lower bound on $\hat{\sigma}_x^2(n)$ can be obtained by assuming that all measurements were collected at the nearest location $x_i$ to $x$. \revone{This is based on the fact that the closer the observation, lower the predictive variance. A mathematical proof for this is provided in Appendix~\ref{app:appendix_one}.} \revone{Let} the nearest measurement location $x_i$ is distance $r$ away from $x$. Assuming that all $n$ measurements were collected at $x_i$, lower bound for posterior variance at $x$ can be calculated using Equation~\ref{posterior_var}, 
\begin{align}\label{eq:varNearest}
\scalebox{0.7}{$
\hat{\sigma}^2_x(n) \geq \sigma_0^2 - 
  \begin{bmatrix}
    k(x, x_i),\ldots, k(x, x_i)
  \end{bmatrix}
  \begin{bmatrix}
    \sigma_0^2+\omega^2 & & \sigma_0^2 \\
    & \ddots &  \\
    \sigma_0^2 & & \sigma_0^2+\omega^2
  \end{bmatrix}^{-1}
  \begin{bmatrix}
    k(x, x_i)    \\
     \vdots  \\
    k(x, x_i)
  \end{bmatrix}$}.
\end{align}
It is worth mentioning that the square matrix in Equation~\ref{eq:varNearest} is of order $n\times n$ since there are $n$ measurements. Substituting the value for $k(x, x_i)=\sigma_0^2\exp\left(-\frac{r^2}{2l^2}\right)$ in Equation~\ref{eq:varNearest} and performing the required matrix operations, we get,
\revone{
\begin{align}\label{eq:varNearestOper}
\begin{split}
\hat{\sigma}^2_x(n) \geq \sigma_0^2 - 
 \frac{\sigma_0^4}{\omega^2}\exp{\left(\frac{-r^2}{l^2}\right)}
  \begin{bmatrix}
     1,\ldots, 1
  \end{bmatrix}\times \\
  \begin{bmatrix}
    1-\frac{1}{n+\frac{\omega^2}{\sigma_0^2}} & & \frac{-1}{n+\frac{\omega^2}{\sigma_0^2}} \\
    & \ddots &  \\
    \frac{-1}{n+\frac{\omega^2}{\sigma_0^2}} & & 1-\frac{1}{n+\frac{\omega^2}{\sigma_0^2}}
  \end{bmatrix}
  \begin{bmatrix}
    1    \\
     \vdots  \\
    1
  \end{bmatrix}.
  \end{split}
\end{align}
Therefore,
\begin{align}
\hat{\sigma}^2_x(n)&\geq\sigma_0^2-\frac{\sigma_0^4}{\omega^2}\left(n\left(1-\frac{1}{n+\frac{\omega^2}{\sigma_0^2}}\right)-\frac{n(n-1)}{n+\frac{\omega^2}{\sigma_0^2}}\right),\\
&\geq\sigma^2_0\left(1-\frac{\exp\left(-\frac{r_{max}^2}{l^2}\right)}{1+\frac{\omega^2}{n\sigma^2_0}}\right).\label{eq:varNearest-trans}
\end{align}}
Even if we had collected infinitely many measurements at the nearest location $x_i$, the posterior variance will still be lower bounded as,
\begin{align}
\hat{\sigma}^2_x(n)& > \lim_{n\to\infty} \sigma^2_0\left(1-\frac{\exp\left(-\frac{r_{max}^2}{l^2}\right)}{1+\frac{\omega^2}{n\sigma^2_0}}\right), \\
&\revone{=} \sigma^2_0\left(1-\exp\left(-\frac{r_{max}^2}{l^2}\right)\right)\label{eq:varNearest-trans-lim}.
\end{align}
If the posterior variance at $x$ even with the infinitely many measurements collected at the nearest measurement location $x_i$ (Equation~\ref{eq:varNearest-trans-lim}) is greater than $\Delta$, \ie,
\begin{align}\label{eq:rmax_Delta}
\begin{split}
        \Delta < \sigma^2_0\left(1-\exp\left(-\frac{r_{max}^2}{l^2}\right)\right) \implies \\  r_{max}  > l\sqrt{-\log\left(1-\frac{\Delta}{\sigma_0^2}\right)},
\end{split}
\end{align}
then it is not possible to bring down the MSE at $x$ below $\Delta$ in any circumstance.
\end{proof}
Next, we prove a sufficient condition that if every point in the environment where no measurement is obtained (\emph{test} location) is sufficiently close to a measurement location, then we can make accurate predictions at each point.
\begin{lemma}[Sufficient Condition] \label{lem:nearvariance}
For a test location $x \in U$, if there exists a measurement location $x_i\in X$, r distance away from $x$ with \revone{$n_\text{suff}$} measurements at $x_i$, such that,
\begin{equation}
r\leq l\sqrt{-\log\left(\left(1+\frac{\omega^2}{\revone{n_\text{suff}}\sigma_0^2}\right)\left(1-\frac{\Delta}{\sigma_0^2}\right)\right)},
\end{equation}
then GP predictions at $x$ will be accurate, \ie, MSE at $x$ will be smaller than $\Delta$.
\end{lemma}
\begin{proof}
From Equation~\ref{eq:varNearest-trans}, we have an expression for variance \revone{of the posterior predictive distribution} at $x$. Taking the other measurement locations into consideration can not increase the posterior variance at $x$. Information never hurts~\cite{cover2012elements}! To prove the sufficiency, we consider \revone{$n_\text{suff}$} measurements at $x_i$ only and discard others knowing that the other locations can not increase the posterior variance at $x$. Bounding the expression in Equation~\ref{eq:varNearest-trans} with $\Delta$ results in,
\begin{align}
\Delta &\geq\sigma^2_0\left(1-\frac{\exp\left(-\frac{r^2}{l^2}\right)}{1+\frac{\omega^2}{\revone{n_\text{suff}}\sigma^2_0}}\right),\\
\exp\left(-\frac{r^2}{l^2}\right)& \geq \left(1+\frac{\omega^2}{\revone{n_\text{suff}}\sigma_0^2}\right)\left(1-\frac{\Delta}{\sigma_0^2}\right), \label{eq:domain-suff}\\
r&\leq l\sqrt{-\log\left(\left(1+\frac{\omega^2}{\revone{n_\text{suff}}\sigma_0^2}\right)\left(1-\frac{\Delta}{\sigma_0^2}\right)\right)}.
\end{align}
\end{proof}
Lemma~\ref{lem:nearvariance} gives a sufficient condition for GP predictions to be accurate at any given test location $x\in U$. The following lemma shows that a finite number of measurements $\revone{n_\text{suff} =} n_{\alpha}$, are sufficient to ensure predictive accuracy in a smaller disk of radius $\frac{1}{\alpha} r_{max}$ around $x_i$, where $\alpha > 1$ (Figure~\ref{figure_disks}).
\begin{figure}
  \centering
  \includegraphics[width=0.5\linewidth]{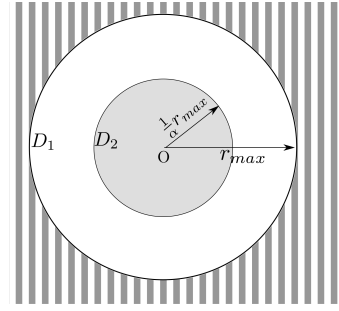}
  \caption{Collecting $n_{\alpha}$ measurements at O suffices to make accurate predictions at all points inside disk $D_2$ (Sufficient condition). No number of measurements at O can ensure predictive accuracy on points outside disk $D_1$ (Necessary condition). \label{figure_disks}}
\end{figure}

\begin{lemma}\label{lem:nalpha}
Given a disk of radius $\frac{1}{\alpha}r_{max}$ centered at $x_i$, $n_{\alpha}$ measurements at $x_i$ suffice to make accurate predictions for all points inside the disk, where,
\begin{equation}\label{eq:nalpha_cond}
n_{\alpha} \geq \left\lceil \frac{\omega^2}{\sigma_0^2}\frac{1}{\left(1-\frac{\Delta}{\sigma_0^2}\right)^{\frac{1}{\alpha^2}-1}-1}\right\rceil.
\end{equation}
\end{lemma}
\begin{proof}
We want a sufficiency condition on the number of measurements $n_{\alpha}$ inside a disk of radius $\frac{1}{\alpha}r_{max}$.
Lemma~\ref{lem:nearvariance} gives an upper bound on the radius of a disk such that all points inside the disk will be accurately predicted after $n_\alpha$ measurements at the center. We construct a disk ($D_2$ in Figure~\ref{figure_disks}), whose radius is equal to $\frac{1}{\alpha}r_{max}$ such that,
\begin{align}\label{eq:alpha_rmax_cond}
&\frac{1}{\alpha}r_{max} \leq l\sqrt{-\log\left(\left(1+\frac{\omega^2}{n_{\alpha}\sigma_0^2}\right)\left(1-\frac{\Delta}{\sigma_0^2}\right)\right)}.
\end{align}
Plugging in the value of $r_{max}$ from Lemma~\ref{lemma_var_tol_exp}, squaring both sides in Equation~\ref{eq:alpha_rmax_cond} and re-arranging for $n_\alpha$ gives the required bound stated in Lemma~\ref{lem:nalpha}. Ceiling function in Equation~\ref{eq:nalpha_cond} accounts for the fact that $n_\alpha$ is an integer.
\end{proof}

A packing of disks of radius $r_{max}$ gives a lower bound on the number of measurements required to ensure predictive accuracy. On the other hand, a covering of disks of radius $\frac{1}{\alpha}r_{max}$ gives us an upper bound on the number of measurements required. To solve Problem~\ref{prob:placement}, what remains is to relate the upper and lower bound and present an algorithm to place the disks of radii $\frac{1}{\alpha}r_{max}$.

\subsection{Placement of Sensors for Problem~\ref{prob:placement}}
We use an algorithm similar to the one presented by Tekdas and Isler~\cite{tekdas2010sensor} \revone{for stationary sensor placement in order to track a target using bearing sensors. In their case, the goal is to place sensors such that irrespective of where the target is in the environment, there are at least three sensors forming a triangle that get good quality bearing information of the target. The show how to cover the environment with disks and place a triangle of sensors within each disk. The setup is different from the one we have; however, we use a similar disk coverage strategy as a subroutine here.} The exact procedure is outlined in Algorithm~\ref{alg:stationary}.

\begin{algorithm}
\caption{\DC \label{alg:stationary}}
\begin{algorithmic}[1]
\Procedure{}{}
\BState \textbf{Input:}  An environment.
\BState \textbf{Output:}  Measurement locations.
\BState \textbf{begin} 
\begin{enumerate}
\item Design a set $\mathcal{X}$ of disks of radii $r_{max}$ which covers the environment and calculate a Maximal Independent Set (MIS) $\mathcal{I}$ of $\mathcal{X}$ greedily \ie, $\mathcal{I}=\textnormal{MIS}(\mathcal{X})$\label{step:mis}. 
\item Place disks of radii $3r_{max}$ concentric with disks in $\mathcal{I}$. Let the set of $3r_{max}$ radii disks is $\bar{\mathcal{X}}$ \label{step:xbar}.
\item Cover each disk in $\bar{\mathcal{X}}$ with disks of radii $\frac{1}{\alpha}r_{max}$ as shown in Figure~\ref{figure_cover} and label centers of all disks of radii $\frac{1}{\alpha}r_{max}$. 
\item Return all the labeled points in previous step as measurement locations.
\end{enumerate}

\BState \textbf{end procedure}
\EndProcedure 
\end{algorithmic}
\end{algorithm}
\begin{theorem}\label{theorem:stat}
\DC{} (Algorithm~\ref{alg:stationary}) gives an $18\alpha^2$-approximation for Problem \ref{prob:placement} in polynomial time.
\end{theorem}
\begin{proof}
Denote the set of measurement locations computed by the optimal algorithm to solve the Problem~\ref{prob:placement} by $X^{*}$. The function MIS in Step~\ref{step:mis} of Algorithm~\ref{alg:stationary} computes a maximally independent set of disks: the disks in $\mathcal{I}$ are mutually non-intersecting (independent) and every disk in $\mathcal{X}\backslash\mathcal{I}$  intersects \revone{at least one} disk in $\mathcal{I}$ (maximal). The set $\mathcal{I}$ can be computed by a simple polynomial greedy procedure: choose an arbitrary disk $d$ from $\mathcal{X}$, add it to $\mathcal{I}$, remove all disks in $\mathcal{X}$ which intersect $d$, and repeat the procedure until no such $d$ exists. 

An optimal algorithm will collect measurements from at least as many measurement locations as the cardinality of $\mathcal{I}$. This can be proved by contradiction. Suppose an algorithm visits measurement locations fewer than the number of disks in $\mathcal{I}$. In that case, there will exist at least one disk of radius $r_{max}$ in $\mathcal{I}$ which will not contain a measurement location. This means that there will be at least a point in that disk which will be more than $r_{max}$ away from each measurement location. From Lemma~\ref{lemma_var_tol_exp}, the robot can never make accurate predictions at that point and hence violating the constraint in Problem \ref{prob:placement}. Hence,
\begin{equation}\label{sensing:loc1}
 |\mathcal{I}| \ \leq \ |X^{*}|.
\end{equation}
Every disk in $\mathcal{X}$ intersects at least one disk in $\mathcal{I}$ and hence, lies within $3r_{max}$ of the center of a disk in $\mathcal{I}$. As a result, $\bar{\mathcal{X}}$ disks cover all the $\mathcal{X}$ disks and hence, the entire environment.\footnote{\revone{Note that $3r_{max}$ is the minimum radius of the bigger disks to guarantee that the entire environment is always covered. In specific instances, it may be possible to cover the environment with smaller than $3r_{max}$ by selecting the MIS using a well-designed heuristic. However, there are environments where $3r_{max}$ will be necessary.}}

Collecting measurements from $18\alpha^2$ locations inside a $3r_{max}$ disk suffice to make accurate predictions in that disk (satisfying the Problem~\ref{prob:placement} constraint for points belonging to that disk) as illustrated in Figure~\ref{figure_cover}. \DC{} collects measurement from $18\alpha^2$ such locations per disk in $\bar{\mathcal{X}}$. It collects measurements from a total of $18\alpha^2|\bar{\mathcal{X}}|$ locations, hence, satisfying the constraint for all points in the area covered by union of $\bar{\mathcal{X}}$ disks. Since, union of $\bar{\mathcal{X}}$ disks covers the entire environment, \DC{} satisfies the constraint for all points in the environment. 
Multiplying both sides of Equation~\ref{sensing:loc1} with $18\alpha^2$, we get, $18\alpha^2  |\mathcal{I}|   \leq  18\alpha^2  |X^{*}|$.
Note that $|\bar{\mathcal{X}}|=|\mathcal{I}|$. Hence,
\begin{align}\label{sensing:loc2}
 18\alpha^2  |\bar{\mathcal{X}}| \ & \leq \ 18\alpha^2 |X^{*}|,  \\
  n_{\DC{}} \ & \leq \ 18\alpha^2  |X^{*}|,
\end{align}
where, $n_{\DC{}}$ is the number of measurement locations for \DC{}. 
\begin{figure}
  \centering
  \includegraphics[width=0.5\linewidth]{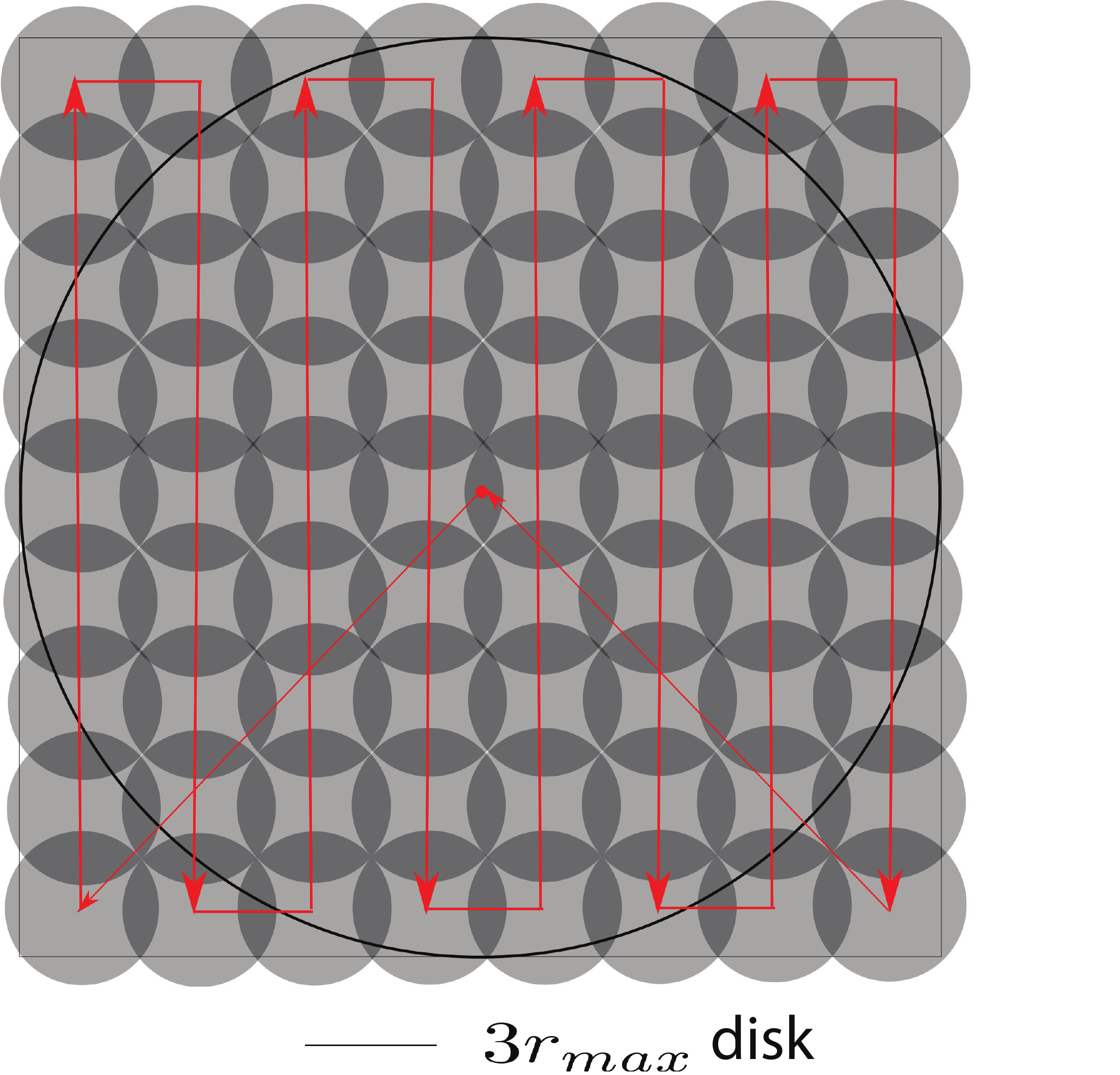}
  \caption{We cover each $3r_{max}$ radius disk with $\frac{1}{\alpha} r_{max}$ radii disks~\revone{(smaller gray disks)} in \revone{lawn-mower} pattern. $18\alpha^2$ disks suffice to cover the bigger disk. The locations of disks of radii $\frac{1}{\alpha}r_{max}$ inside a disk of radius $3r_{max}$ are obtained by covering the square circumscribing bigger disk with smaller squares inscribed in smaller disks. The centers of smaller squares coincide with the centers of smaller disks. \label{figure_cover}}
\end{figure}
\end{proof}

\subsection{Finding an Approximate Optimal Trajectory for Problem~\ref{prob:mobile}}
The algorithm for Problem~\ref{prob:mobile} builds on the algorithm presented in the previous section. The locations where measurements are to be made become the locations that are to visited by the robot. The robot must obtain at least $n_\alpha$ measurements at the center of each disk of radius $\frac{1}{\alpha}r_{max}$. A pseudo-code of the algorithm is presented in the Algorithm~\ref{alg:moblie}.
\begin{algorithm}
\caption{\DCT \label{alg:moblie}}
\begin{algorithmic}[1]
\Procedure{}{}
\BState \textbf{Input:} A set of measurement locations calculated from Algorithm~\ref{alg:stationary}.
\BState \textbf{Output:} An approximate optimal tour visiting all the measurement locations.
\BState \textbf{begin} 
\begin{enumerate}
\item Calculate approximate TSP tour visiting centers of the $3r_{max}$ radius disks (set $\bar{\mathcal{X}}$) disks. \label{step:tour}
\item Cover $\bar{\mathcal{X}}$ disk containing the starting location in \revone{lawn-mower} pattern visiting the centers of corresponding disks of radius $\frac{1}{\alpha}r_{max}$ and make $n_\alpha$ measurements at each center point. \label{step:indi}
\item Move to the center of next $\bar{\mathcal{X}}$ disk along the tour calculated in Step~\ref{step:tour}.\label{step:movement}
\item Repeat Steps~\ref{step:indi} and~\ref{step:movement} until all $\bar{\mathcal{X}}$ disks are covered.
\end{enumerate}
\BState \textbf{end procedure}
\EndProcedure
\end{algorithmic}
\end{algorithm}

\begin{theorem}
\DCT{} (Algorithm~\ref{alg:moblie}) yields a constant-factor approximation algorithm for Problem~\ref{prob:mobile} in polynomial time.
\end{theorem}
\begin{proof}
From Theorem~\ref{theorem:stat}, we have a constant approximation bound on number of measurement locations. Let the time (travel and measurement time) taken by the optimal algorithm be $T^{*}_1$. Using notation from Theorem~\ref{theorem:stat}, we assume that the optimal traveling salesperson with neighborhoods (TSPN) time to visit disks in $\mathcal{I}$ be $T^{*}_{\mathcal{I}}$.  In TSPN, we are given a set of geometric neighborhoods, and the objective is to find the shortest tour that visits at least one point in each neighborhood (disks in this case)~\cite{tokekar2016sensor}. The optimal algorithm will visit at least all disks once in $\mathcal{I}$ which gives the following minimum bounds on the optimal travel time ($T^{*}_{travel}$) and optimal measurement time ($T^{*}_{measure}$),   
\begin{align} \label{inequality_tspn}
T^{*}_{\mathcal{I}}   \leq  T^{*}_{travel}; \eta|\mathcal{I}|   \leq  T^{*}_{measure}.
\end{align}
Let the optimal time to visit the centers of disks in $\mathcal{I}$ be $T^{*}_{\mathcal{I}^C}$. An upper bound on $T^{*}_{\mathcal{I}^C}$ can be established by the fact that upon visiting each disk, the robot can visit the center of that disk and return back by adding an extra tour length of $2r_{max}$, \ie, a detour of maximum length $|\mathcal{I}| \times 2 r_{max}$ for all disks in $\mathcal{I}$. As a result: $T^{*}_{\mathcal{I}^C} \leq T^{*}_{\mathcal{I}} + 2r_{max}|\mathcal{I}|$.
Using inequality from Equation~\ref{inequality_tspn}: $T^{*}_{\mathcal{I}^C}  \leq T^{*}_{travel} + 2 r_{max} |\mathcal{I}|$.
For any disk in $\bar{\mathcal{X}}$, the length of \revone{lawn-mower} path starting from the its center and return back (Figure~\ref{figure_cover}) after visiting all center points of $\frac{1}{\alpha}r_{max}$ disks will be of order $\mathcal{O}(\alpha^2)r_{max}$. Hence, the total travel time for \DCT{} is: $T_C+|\mathcal{I}|\mathcal{O}(\alpha^2)r_{max}$, where $T_C$ is the $(1+\epsilon)$-approximated time with respect to the optimal TSP tour returned by the $(1+\epsilon)$-approximation algorithm to visit the centers of the disks in $\bar{\mathcal{X}}$ (or $\mathcal{I}$ disks since they are concentric).
$T_C$ can be calculated in polynomial time~\cite{646145} having bounds: $T_C   \ \leq \ (1+\epsilon) \ T^{*}_{\mathcal{I}^{C}}$,
with $T^{*}_{\mathcal{I}^{C}}$ being the optimal TSP time to visit the centers of $\bar{\mathcal{X}}$ disks. Measurement time for \DCT{} is $18\alpha^2\eta n_2|\mathcal{I}|$.
Hence, the total time $T_{alg}^1$ for \DCT{} is,
\begin{align}
 T_{alg}^1 &= T_C+|\mathcal{I}|\mathcal{O}(\alpha^2)r_{max} + 18\alpha^2\eta n_2|\mathcal{I}| \\
 &\leq (1+\epsilon) \ T^{*}_{\mathcal{I}^{C}}+\mathcal{O}(\alpha^2)r_{max}|\mathcal{I}| + 18\alpha^2\eta n_2|\mathcal{I}|,\\
\begin{split}
& \leq (1+\epsilon) \ \left(T^{*}_{travel} + 2 r_{max} |\mathcal{I}|\right)+\mathcal{O}(\alpha^2)r_{max}|\mathcal{I}| \\
              &\qquad \hspace{1.3in} + 18\alpha^2\eta n_2|\mathcal{I}|.
\end{split}
\label{eq:costBound}
\end{align}
\revone{$n_2$ is the number of sufficient measurements required inside a disk of radius $\frac{1}{2}r_{max}$ (Lemma~\ref{lem:nalpha} with $\alpha=2$).} Length of any tour that visits $k$ non-overlapping equal size disks of radii $r$ is at least $0.24 k r$~\cite{tekdas2012efficient}, which gives $0.24r_{max} |\mathcal{I}|  \leq T^{*}_{\mathcal{I}}$. Combining this result with Equation~\ref{inequality_tspn} modifies the bounds in Equation~\ref{eq:costBound} as,
\begin{align}
\begin{split}
 T_{alg}^1 & \leq \left((1+\epsilon)\left(1+\frac{2}{.24}\right)+\frac{\mathcal{O}(\alpha^2)}{.24}\right)T^{*}_{travel} \\
              &\qquad \hspace{1.4in} + 18\alpha^2n_2 T^{*}_{measure},
\end{split}
 \\
&\scalebox{0.95}{$\leq \textnormal{max}\left(9.33(1+\epsilon)+\frac{82}{.24}, 72n_2\right) \left(T^{*}_{travel}+T^{*}_{measure}\right),\label{eq:maxRHS}$} \\
& \leq \textnormal{max}\left(9.33(1+\epsilon)+\frac{\mathcal{O}(\alpha^2)}{.24}, 18\alpha^2 n_2\right) T^{*}_1 \\
& \leq c T^{*}_1,\label{eq:finalalg2}
\end{align}
where $c$, a constant, is larger one of the two quantities inside the bracket in Equation~\ref{eq:finalalg2}.
\end{proof}
\revone{Note that the Algorithm~\ref{alg:moblie} collects same number of measurements from each measurement location. There may be another algorithm that collects different number of measurements from different locations which may result in better performance. This modification is an avenue for future work.}

\subsection{Finding an Approximate Optimal Trajectory for Problem~\ref{prob:multi-robot}}
When one robot can not handle a large territory, to speed up the task, $k$ robots can be sent to collectively visit1 all measurement locations. A natural objective is to ensure that no robot has too large of a task. Hence, We choose our optimization criterion as minimizing the maximum of the $k$-robot tour costs. This is equivalent to minimizing the time taken by the last robot to return back to the common starting location. Our proposed algorithm only works if the robots start and return back to the same location -- called depot. Any measurement location can be chosen as the depot but in our case, we assume that the robots start from and return back to a pre-defined depot.

We now describe an algorithm which employs a tour-splitting heuristic to plan for $k$ robots. We modify the heuristic proposed by Frederickson et al.~\cite{4567906} to account for the measurement time, and not just the travel time.

Let the output tour of the robot from Algorithm~\ref{alg:moblie} be denoted by $\tau$ and $l_{max}$ be the distance of farthest measurement location from the depot.
\begin{algorithm}
\caption{\kDCT \label{alg:multimoblie}}
\begin{algorithmic}[1]
\Procedure{}{}
\BState \textbf{Input:} Tour calculated from Algorithm~\ref{alg:moblie}, Depot location $x_1$.
\BState \textbf{Output:}  $k$ approximate optimal paths visiting all measurement locations collectively.
\BState \textbf{begin}
\begin{enumerate}
\item For $j^{th}$ robot, $1\leq j <k$, find the last measurement location $\revone{x}_{p(j)}$ such that the time taken to travel from $\revone{x}_1$ to $\revone{x}_{p(j)}$ along $\tau$ is not greater than $\frac{j}{k}\left(T_{alg}^1-\left(2l_{max}+\eta n_2\right)\right)+\left(l_{max}+\eta n_2\right)$.
\item Obtain \textit{k} subtours as $R_1=(x_1,\ldots,\revone{x}_{p(1)},\revone{x}_1)$, $R_2=(\revone{x}_1,\revone{x}_{p(1)+1},\ldots,\revone{x}_{p(2)},\revone{x}_1)$, $\ldots$ $R_k=(\revone{x}_1,\revone{x}_{p(k-1)+1}\ldots,\revone{x}_n,\revone{x}_1)$.
\end{enumerate}

\BState \textbf{end procedure}
\EndProcedure
\end{algorithmic}
\end{algorithm}
\begin{theorem}\label{theorem:multi}
\kDCT{} (Algorithm~\ref{alg:multimoblie}) yields a $(c+2)$ approximation algorithm for Problem~\ref{prob:multi-robot} in polynomial time, given a $c$-approximation algorithm for Problem~\ref{prob:mobile}.
\end{theorem}
\begin{proof}
First, we prove that the time taken along every subtour is bounded and eventually show that the bound is within a constant factor of the optimal time. With $k$ robots, let the subtours for $1^{st}$ and $k^{th}$ robot are $\revone{x}_1 \rightarrow \revone{x}_{p(1)}\xrightarrow{} \revone{x}_1$ and $\revone{x}_1 \rightarrow \revone{x}_{p(k-1)+1}\rightarrow \revone{x}_1$ respectively (an example with $k=5$ is shown in Figure~\ref{fig:tour_split}). Subtours for the remaining robots can be denoted by $\revone{x}_1 \rightarrow \revone{x_{p(j-1)+1}} \rightarrow \revone{x}_{p(j)}\rightarrow \revone{x}_1$, where $1<j<k$.

Substituting $j=1$ in Algorithm~\ref{alg:multimoblie}, the time to travel from $\revone{x}_1$ to $\revone{x}_{p(1)}$ along $\tau$, $T(\revone{x}_1\xrightarrow{\tau} \revone{x}_{p(1)})$ is no greater than $\frac{1}{k}\left(T_{alg}^1-\left(2l_{max}+\eta n_2\right)\right)+\left(l_{max}+\eta n_2\right)$. Time for the first subtour is hence bounded by $T(\revone{x}_1\xrightarrow{\tau} \revone{x}_{p(1)})+T(\revone{x}_{p(1)}\xrightarrow{} \revone{x}_1)$, \ie, $\frac{1}{k}\left(T_{alg}^1-\left(2l_{max}+\eta n_2\right)\right)+\left(l_{max}+\eta n_2\right)+l_{max}$.
From the condition in Algorithm~\ref{alg:multimoblie}, we know that $\revone{x}_{p(k-1)}$ is the last location such that,
\begin{align}
\begin{split}
    T\left(\revone{x}_1\xrightarrow{\tau} \revone{x}_{p(k-1)}\right) \leq \frac{k-1}{k}\left(T_{alg}^1-\left(2l_{max}+\eta n_2\right)\right)+ \\ \left(l_{max}+\eta n_2\right),
\end{split}
\end{align}
and hence,
\begin{align}
\begin{split}
T\left(\revone{x}_1\xrightarrow{\tau} \revone{x}_{p(k-1)+1}\right) \geq \frac{k-1}{k}\left(T_{alg}^1-\left(2l_{max}+\eta n_2\right)\right)+ \\ \left(l_{max}+\eta n_2\right).
\end{split}
\end{align}
Subtracting both sides from $T^1_{alg}$,
\begin{align}
\begin{split}
T^1_{alg}-T\left(\revone{x}_1\xrightarrow{\tau} \revone{x}_{p(k-1)+1}\right) \leq & T^1_{alg}- \frac{k-1}{k}\Big(T_{alg}^1-(2l_{max}\\
&+\eta n_2)\Big)+ \left(l_{max}+\eta n_2\right),
\end{split}
\end{align}
which gives,
\begin{equation}
\scalebox{0.85}{$T\left(\revone{x}_{p(k-1)+1}\xrightarrow{\tau} \revone{x}_1\right) \leq \frac{1}{k}\Big(T_{alg}^1-(2l_{max}
+\eta n_2)\Big)+ \left(3l_{max}+2\eta n_2\right)$},
\end{equation}
and hence, time for the last subtour is bounded by $T\left(\revone{x}_1\xrightarrow{}\revone{x}_{p(k-1)+1} \right)+T\left(\revone{x}_{p(k-1)+1}\xrightarrow{\tau} \revone{x}_1\right)$, \ie, $\frac{1}{k}\Big(T_{alg}^1-(2l_{max}+\eta n_2)\Big)+ \left(3l_{max}+2\eta n_2\right)+l_{max}$. Similar inequalities can be derived for remaining subtours as follows.
For $1\leq j\leq k-2$, following inequalities hold from Algorithm~\ref{alg:multimoblie}, 
\begin{equation}
\scalebox{0.9}{$T(\revone{x}_1\xrightarrow{\tau} \revone{x}_{p(j)+1}) \geq \frac{j}{k}\left(T_{alg}^1-\left(2l_{max}+\eta n_2\right)\right)+\left(l_{max}+\eta n_2\right),\label{eq:lowerBound}$}
\end{equation}
\begin{equation}
\scalebox{0.9}{$T(\revone{x}_1\xrightarrow{\tau} \revone{x}_{p(j+1)}) \leq \frac{j+1}{k}\left(T_{alg}^1-\left(2l_{max}+\eta n_2\right)\right) +\left(l_{max}+\eta n_2\right). \label{eq:upperBound}$}
\end{equation}
Subtracting Equation~\ref{eq:lowerBound} from \ref{eq:upperBound} results in,
\begin{align}
    T(\revone{x}_{p(j)+1}\xrightarrow{\tau} \revone{x}_{p(j+1)})\leq \frac{1}{k} \left(T_{alg}^1-\left(2l_{max}+\eta n_2\right)\right),
\end{align}
\ie, the time taken along remaining subtours, $T\left(\revone{x}_1 \rightarrow \revone{x}_{p(j-1)+1} \xrightarrow{\tau} \revone{x}_{p(j)}\rightarrow \revone{x}_1\right)$, where $1<j<k$, is also bounded by $\frac{1}{k} \left(T_{alg}^1-\left(2l_{max}+\eta n_2\right)\right)+2l_{max}$.
Hence, we can conclude that \revone{the} time taken along each subtour does not exceed $\frac{1}{k}\left(T_{alg}^1-\left(2l_{max}+\eta n_2\right)\right)+\left(4l_{max}+2\eta n_2\right)$. 

Let $T_{alg}^k$ be the time taken for largest of the \textit{k} subtours generated by the Algorithm~\ref{alg:multimoblie}, and $T^{*}_k$ be the cost of the largest subtour in an optimal solution to Problem~\ref{prob:multi-robot}.
We have,
\begin{align}
    T_{alg}^k &\leq  \frac{1}{k}\left(T_{alg}^1-\left(2l_{max}+\eta n_2\right)\right)+\left(4l_{max}+2\eta n_2\right) \\
& \leq\frac{T_{alg}^1}{k}+\left(2l_{max}+\eta n_2\right)\left(2-\frac{1}{k}\right).
\end{align}
From the triangle inequality, $T^{*}_k\geq \frac{1}{k}T^{*}_1$. It is natural to think that at least one robot will have to go to the farthest location from the depot and come back from there after collecting $n_2$ measurements which gives us a lower bound on the output of the optimal algorithm, \ie, $2l_{max}+\eta n_2\leq T_k^{*}$. Combining these results with Equation~\ref{eq:finalalg2}, we get,
\begin{align}
 T_{alg}^k & \leq \frac{c}{k}T^{*}_1+T^{*}_k\left(2-\frac{1}{k}\right)\\
&\leq \left(c+2-\frac{1}{k}\right)T^{*}_k \\
&\leq \left(c+2\right)T^{*}_k.
\end{align}

\begin{figure}
  \centering
  \includegraphics[width=0.4\textwidth]{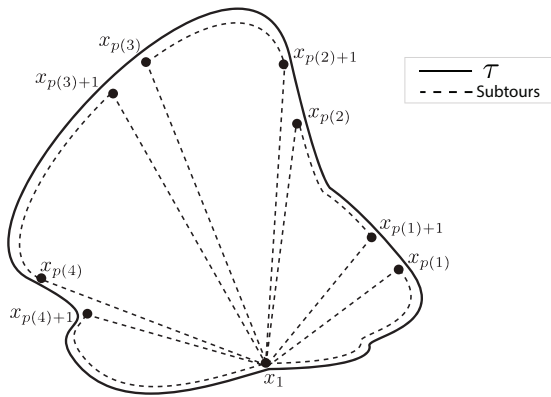}
  \caption{Splitting the tour for one robot ($\tau$) into $5$ subtours. \revone{The solid line shows an initial single robot tour $\tau$ starting and ending at $x_1$. The dotted lines denote the individual robot subtours starting and ending at $x_1$ obtained by splitting the single tour $\tau$.}
\label{fig:tour_split}}
\end{figure}
\end{proof}

\section{Empirical Evaluation} \label{ch:results}
\revone{In this section, we report results from empirical evaluation of the theoretical results. We show qualitative and quantitative comparison of our algorithms with other baseline strategies through simulations using precision agriculture as our motivating example.}

\paragraph*{Dataset}
\revone{We use a real-world dataset~\cite{20023117691}, collected from a farm, consisting} of organic matter (OM) measurements manually collected from several hundred locations within the farm. The maximum and minimum values of the underlying field are $54.6$~\textit{parts per million (ppm)} and $25.4$ \textit{ppm} respectively shown by the colorbar (Figure~\ref{fig:actual_OM}). \revone{Taking this into account, we set $\Delta$ to be equal to $4$ which is $10\%$ of the average of maximum and minimum field values. We use a simulated sensor that returns a noisy version of the ground truth measurement with an additive Gaussian noise of variance, $\omega = 0.0361$. }

The squared-exponential kernel has three hyperparameters: length scale $(l)$, signal variance $(\sigma^2_0)$, and noise variance ($\omega^2$). \revone{The values of $l$, $\sigma_0$, and $\omega^2$ were estimated to be $8.33$ meter, $12.87$, and $0.0361$ respectively} by minimizing the negative log-marginal likelihood of the manually collected data. We assume that the estimated values are the true values of the kernel hyperparameters. \revone{In a general application where some prior data is available, the hyperparameters can be estimated in a similar way. We used the GPML toolbox to perform the necessary GP operations~\cite{rasmussen2010gaussian}.}

\subsection{Qualitative Example}
\paragraph*{Stationary Sensor Placement}
The final predicted OM content after performing inference using the measurements obtained is shown in Figure~\ref{predicted:final}. \revone{This predicted OM content is the average of ten trials. In each trial, the reported value by the sensor can be different even at the same location because of the simulated noise. Figure~\ref{fig:prediction_Error} shows a plot of the prediction error averaged over those ten trials.} \revone{We observe that the average prediction error is below $\Delta=4~ppm$ at each location in the environment. It is important to mention that average prediction error is not same as the MSE. The MSE at a location is the expected squared error in prediction at that location. The average prediction error,  referred as empirical MSE in the following text, is an empirical estimate of that expectation. As the number of runs increases, the empirical MSE will converge to the actual MSE. We verify it through simulations and report the results later. Our theoretical guarantees hold for the MSE and not for the empirical MSE. However, one can expect that the empirical MSE will also be less than the pre-defined threshold $\Delta$ given enough trials.} The regions where the OM content changes sharply tend to be more erroneously predicted as shown by the lighter colored regions in Figure~\ref{fig:prediction_Error}. This can be attributed to the inherent smoothness assumptions of a squared-exponential kernel.
\begin{figure*}
\centering
\subfigure[The actual organic matter content \revone{(\textit{ppm})}.]{\label{fig:actual_OM}\includegraphics[width=0.32\textwidth]{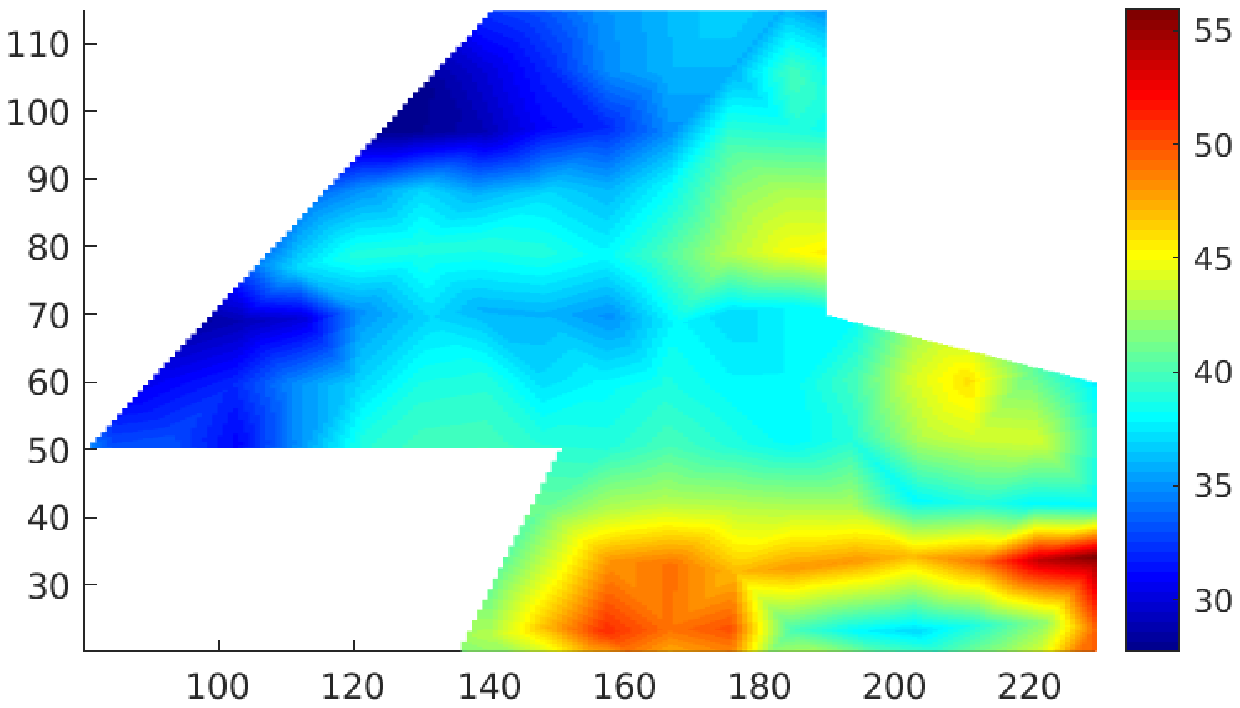}}
\subfigure[The predicted organic matter content \revone{(\textit{ppm})}.]{\label{predicted:final}\includegraphics[width=0.32\textwidth]{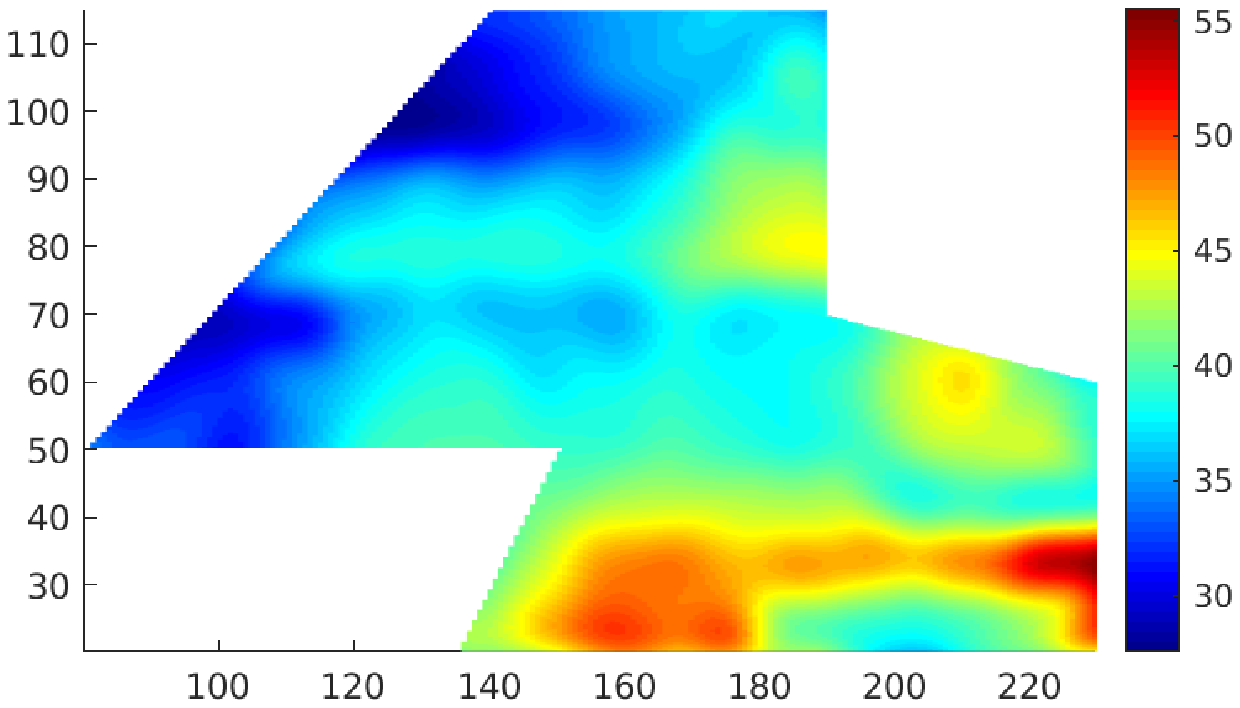}}
\subfigure[Prediction error between the actual and the predicted OM content \revone{(\textit{ppm})}.]{\label{fig:prediction_Error}\includegraphics[width=0.32\textwidth]{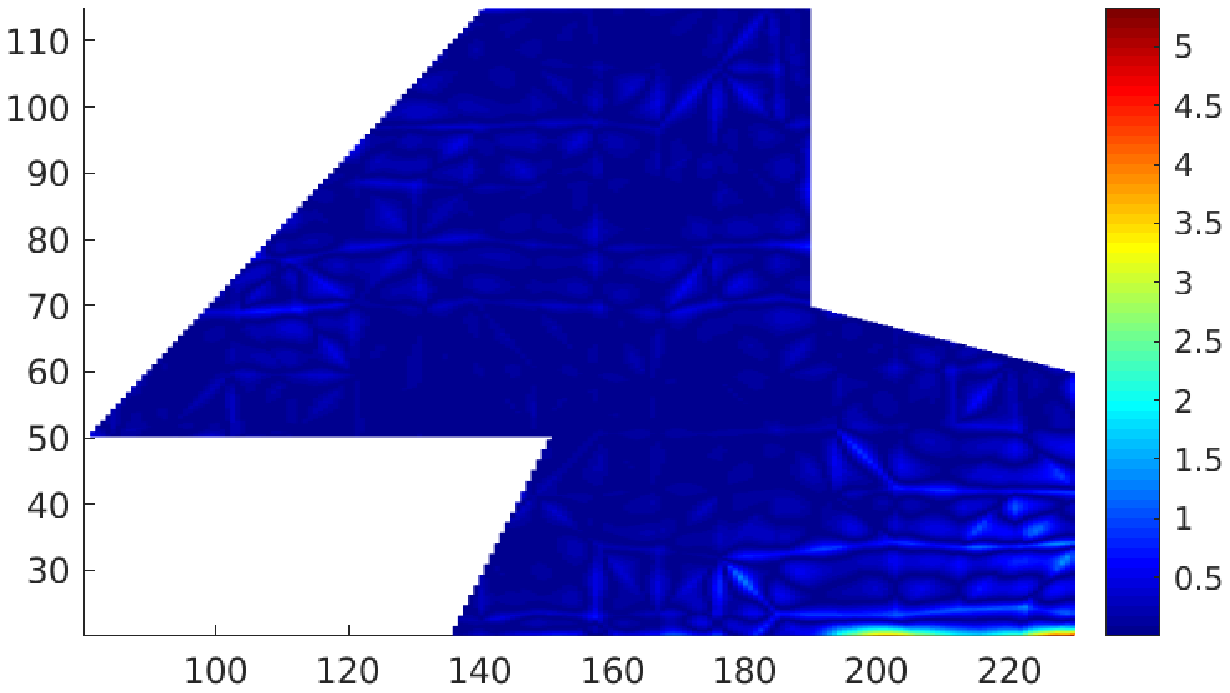}}
\caption{Actual and predicted OM content comparison. The farm is shown as the colored region with the colorbar denoting concentrations at different locations. All distance units are in \textit{meter}.~\label{fig:both_fields}}
\end{figure*}

\paragraph*{Single and Multi-Robot Tours}
The measurement locations computed by the \DCT{} are shown in Figure~\ref{fig:measurement_Locs}. \revone{As post-processing, we removed the redundant measurement locations in overlapping $3r_{max}$ radii disks. After performing this step, the total number of measurement locations was $2320$.} A covering of the farm with disks of radii $3r_{max}$ and an approximate optimal tour visit the centers of those disks calculated by \DCT{} is shown in Figure~\ref{disk_cover}. We compute the optimal TSP tour since this is a reasonably sized instance. The lawn-mower detours visiting individual $3r_{max}$ disks have been omitted to make the figure more legible. For the multi-robot version, we assume that we have three robots. Splitting of a single robot tour (Figure~\ref{disk_cover}) in three subtours is shown in Figure~\ref{fig:multi_split}. The robots start from a common depot.

\begin{figure*}
\centering
\subfigure[Measurement locations calculated by the \DC{} algorithm. \revone{There are $1012$ measurement locations.}]{\label{fig:measurement_Locs}\includegraphics[width=0.32\textwidth, height = 0.23\textwidth]{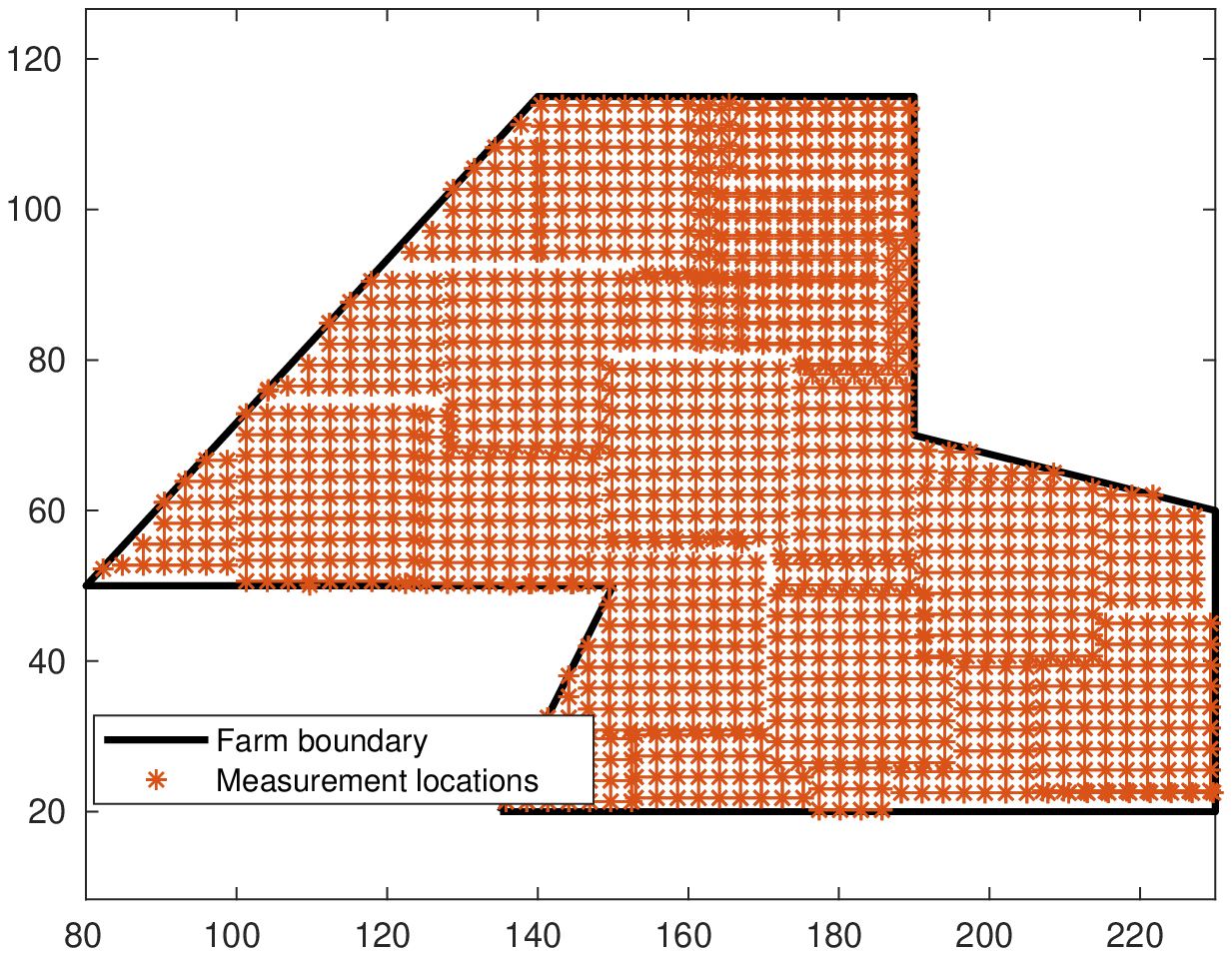}}
\subfigure[The red disks are of radii $3r_{max}$ which are concentric with disks of radii $r_{max}$ in $\mathcal{I}$. The depot is denoted by O.]{\label{disk_cover}\includegraphics[width=0.32\textwidth]{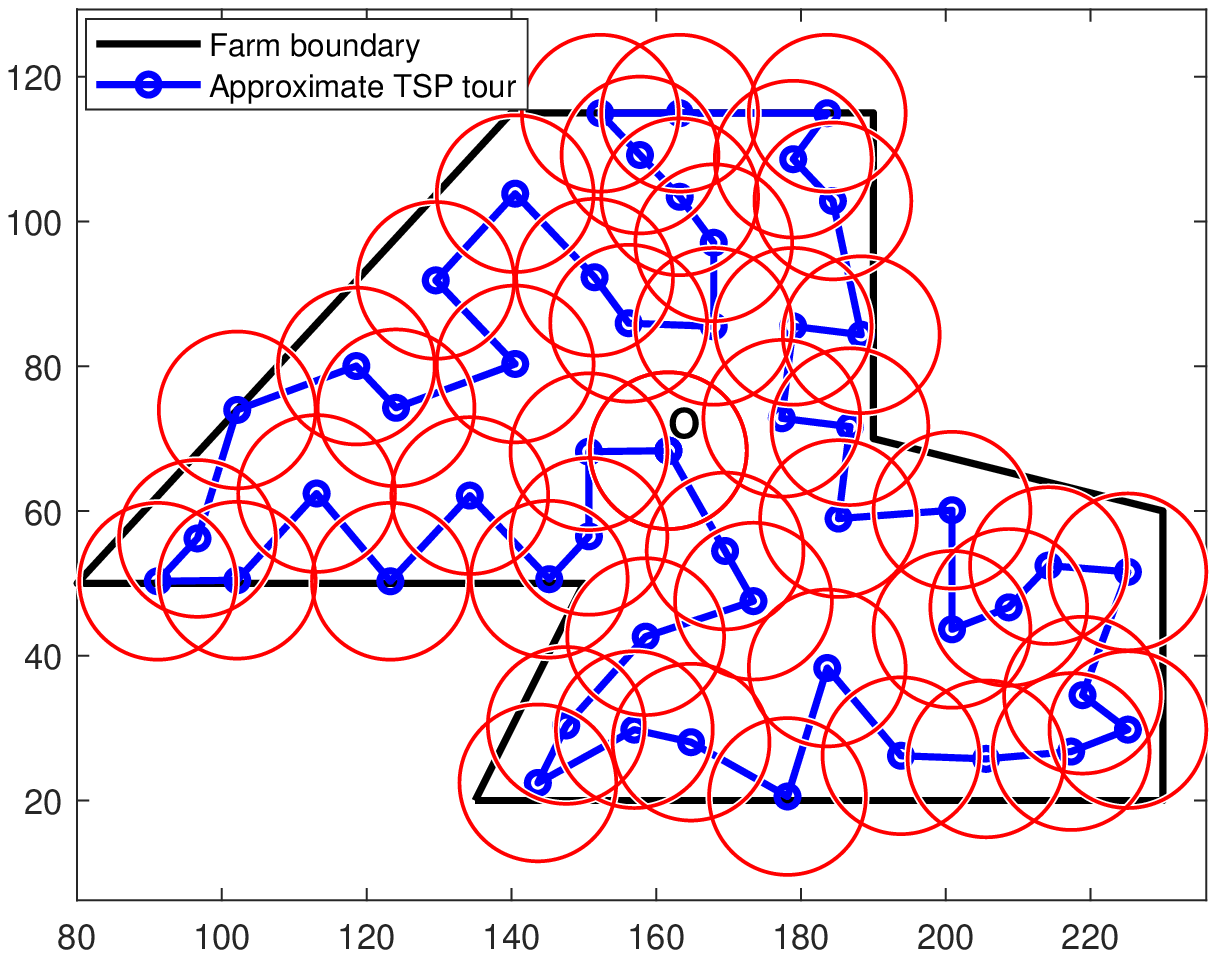}}
\subfigure[All robots start from and return to the depot after making measurements.]{\label{fig:multi_split}\includegraphics[width=0.32\textwidth]{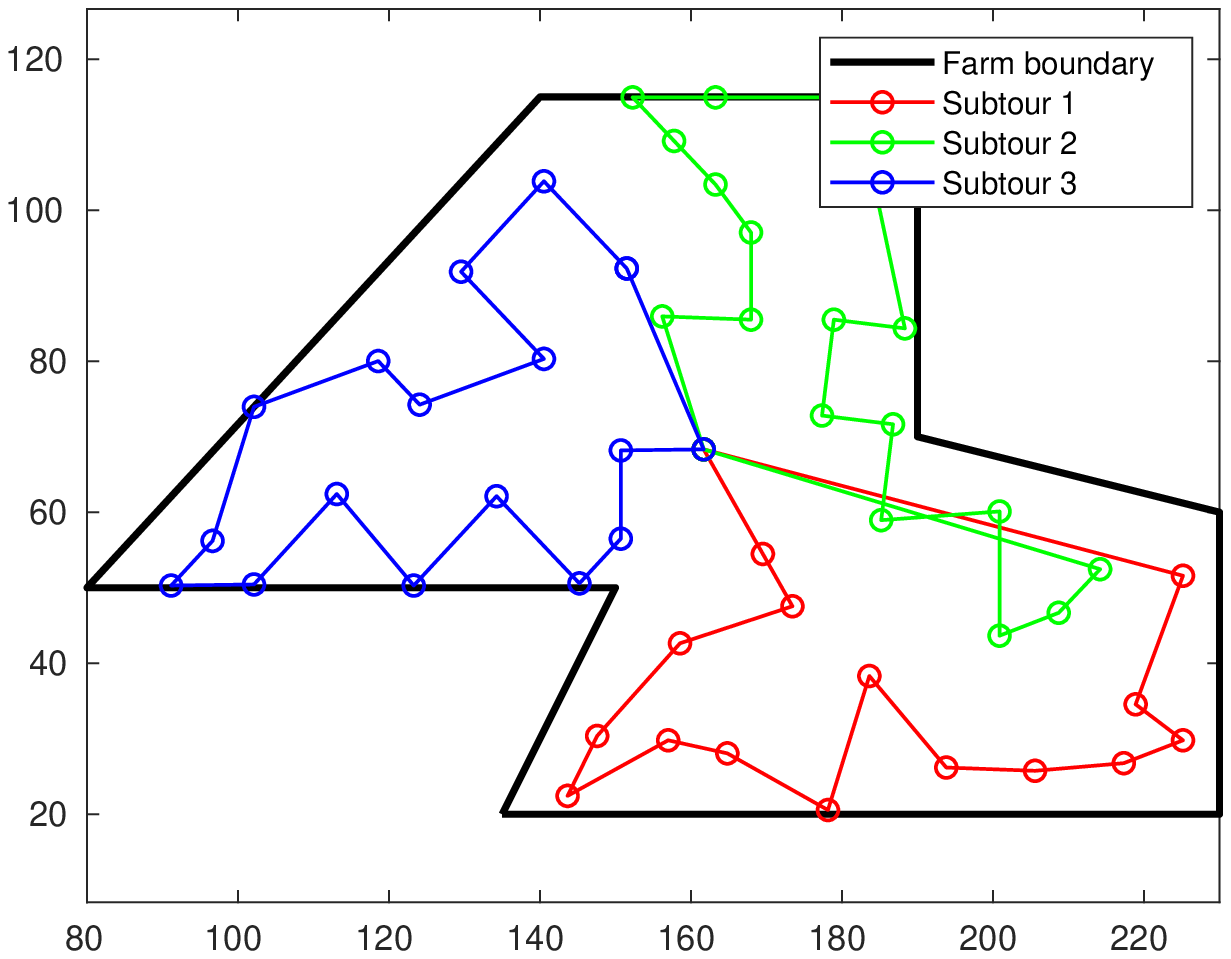}}
\caption{Measurement locations and the tours computed by \DC{} and \kDCT{}. For Figures~\ref{disk_cover} and~\ref{fig:multi_split}, the complete tours that take detours to visit all the locations in Figure~\ref{fig:measurement_Locs}  have been omitted to make the figures more legible.}
\end{figure*}

\paragraph*{Varying values of $\Delta$}
\revone{In some applications, one may be interested in having more accurate predictions in some parts of the environment than others. Our algorithms provides a way to choose locations and plan paths in such applications as well. To demonstrate this, we divide the farm in three sub-environments that have different $\Delta$ tolerances. The left-most, middle, and right-most regions have thresholds of $\Delta=6,~4~\textnormal{and}~2~ppm$ respectively. We solve for the measurement locations independently in each region. The corresponding $r_{max}$ values were calculated to be $4.97,~3.93,~\textnormal{and}~2.70$ meters respectively using Equation~\ref{eq:rmax_Delta}. Figure~\ref{fig:diffDeltaMeasure} shows the measurement locations. One can qualitatively observe that the algorithm places fewer measurement locations in the left-most sub-environment which allows for the highest error tolerance. 

An approximate TSP tour visiting the centers of all $3r_{max}$ disks, in all three regions, is shown in Figure~\ref{fig:diffDeltaTour}. The size of the disks shrinks as one moves to the right-most sub-region which has the least tolerance for prediction error. The TSP tour goes outside the environment in this case, which may be feasible if an aerial robot is used to monitor the farm. In case of applications, where the robot must stay inside the environment, we can enforce this constraint by replacing the Euclidean edge weights in the TSP input graph with the length of the shortest path between two vertices inside the environment.}

\begin{figure}
    \centering
  \includegraphics[width=0.4\textwidth]{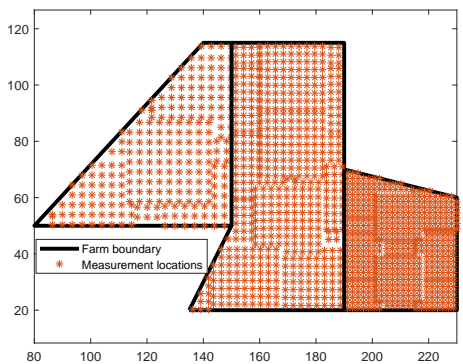}
  \caption{Measurement locations for different values of $\Delta$.~\label{fig:diffDeltaMeasure}.}
\end{figure}

\begin{figure}
    \centering
  \includegraphics[width=0.4\textwidth]{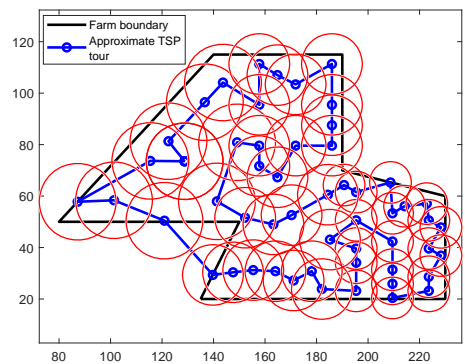}
  \caption{An approximate TSP tour to visit the $3r_{max}$ disks. $r_{max}$ values depend on $\Delta$ (Equation~\ref{eq:rmax_Delta}) and hence, vary in different $\Delta$ sub-regions. Note the shrinking size of disks as one moves towards right.~\label{fig:diffDeltaTour}}.
\end{figure}

\subsection{Comparisons with Pre-defined Lawn-mower Tours}
\revone{One can observe from Figure~\ref{fig:measurement_Locs} that the measurement location pattern closely resembles a lawn-mower pattern. It motivated us to compare the performance of our algorithms and with lawn-mower plan. Figure~\ref{fig:postLawnVSDCT} and~\ref{fig:predictLawnVSDCT} show the average posterior variance and average empirical (for ten trials) MSE respectively for a pre-defined lawn-mower pattern with varying grid resolutions on a semi-logarithmic scale. Note that the posterior variance at a test location is always same in each trial because it is not a function of the actual measurement value. The blue horizontal line corresponds to DCT and is shown for the sake of comparison. 

A plot of the time taken by the robot to cover lawn-mower patterns with various grid resolutions is shown in Figure~\ref{fig:timeLawn}. The lawn-mower lines in Figures~\ref{fig:postLawnVSDCT},~\ref{fig:predictLawnVSDCT}, and~\ref{fig:timeLawn} intersect the DCT lines at approximately a resolution of $2$ meters. It suggests that one would need to create a grid of approximately that resolution to achieve same performance as DCT. Figure~\ref{fig:LawnVSDCT} shows the average posterior variance for DCT and a pre-defined lawn-mower of resolution $2.4$ meter as a function time elapsed along a deployment (averaged over 10 deployments). We chose a resolution of $2.4$ meters since a lawn-mower planner with this resolution has approximately the same number of measurement locations as \DCT{}. We observe that both perform almost the same empirically. 

One may wonder why we cannot simply use the lawn-mower pattern, instead of \DCT{}. To create a lawn-mower pattern, one would need to pick a grid resolution. There is no systematic way of picking this resolution without enumerating a few combinations to analyze the trade-off between time and posterior variance or MSE. This can be wasteful. Instead, we present a systematic way of planning the measurement locations and give explicit theoretical guarantees on time and MSE or variance.}

\begin{figure}
    \centering
  \includegraphics[width=0.4\textwidth]{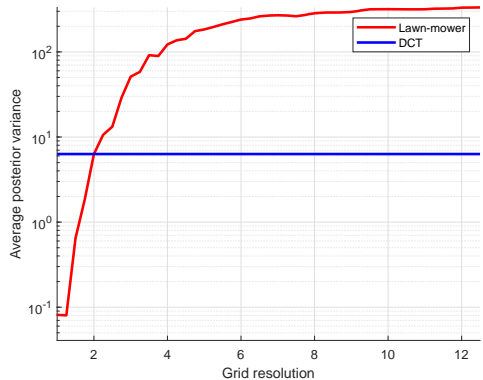}
  \caption{Average posterior variance for varying degree of lawn-mower resolutions.~\label{fig:postLawnVSDCT}}.
\end{figure}

\begin{figure}
    \centering
  \includegraphics[width=0.4\textwidth]{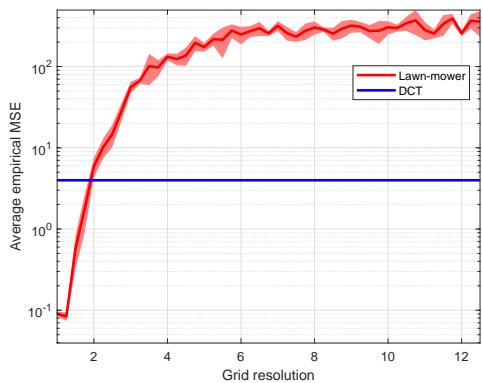}
  \caption{Average empirical MSE for varying degree of lawn-mower resolutions.~\label{fig:predictLawnVSDCT}}
\end{figure}

\begin{figure}
    \centering
  \includegraphics[width=0.4\textwidth]{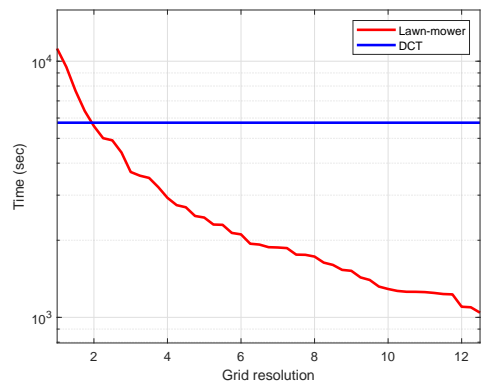}
  \caption{Time spent by the robot with lawn-mower planners of different grid resolutions.~\label{fig:timeLawn}}
\end{figure}

\begin{figure}
    \centering
  \includegraphics[width=0.4\textwidth]{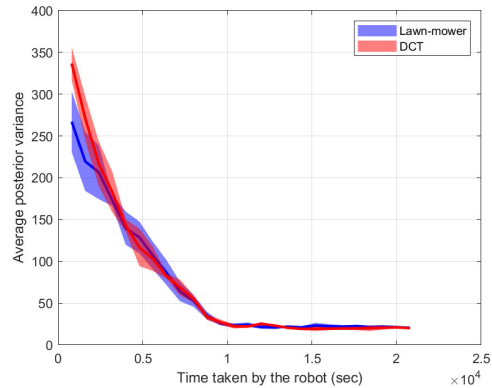}
  \caption{Average posterior variance as a function of time spent by the robot.~\label{fig:LawnVSDCT}}
\end{figure}

\subsection{Comparison With Other Baselines}
A comparison between \DCT{} and two baselines, entropy-based, and MI-based planner is shown in Figure~\ref{fig:compBaselines}. The measurement locations for the entropy-based and MI-based planners were calculated greedily, \ie, picking the next location at the point of maximum entropy and MI respectively as described in~\revone{\cite{krause2008near}}. \revone{We study the average posterior variance and average empirical MSE in prediction as a function of the total time (measurement plus traveling) spent by the robot on the farm for each planner. After finding the measurement locations for each planner separately, TSP tours visiting those locations were calculated. The $X$ axis in Figure~\ref{fig:compBaselines} shows the time taken along a tour and the $Y$ axis shows the respective metrics based on measurements collected until that point in time along the tour (averaged over ten trials).} We observe that \DCT{} performs at par with other planners. The entropy-based planner results in the most significant reduction in posterior variance and average empirical MSE initially. This can be explained by the fact that the entropy-based planning tends to spread the measurement locations far from each other resulting in covering a bigger portion of the environment initially. However, \DCT{} converges to a lower value of average empirical MSE and average posterior variance.

\begin{figure*}
\centering
\subfigure[Average empirical MSE]{\label{fig:avg_empError}\includegraphics[width=0.4\textwidth]{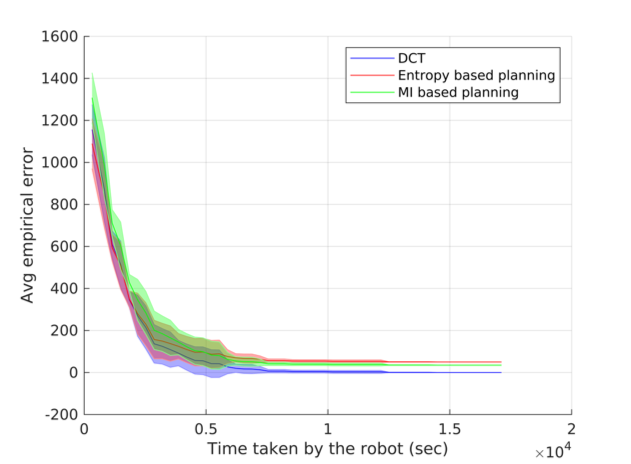}}
\subfigure[Average posterior variance]{\label{fig:post_Var}\includegraphics[width=0.4\textwidth]{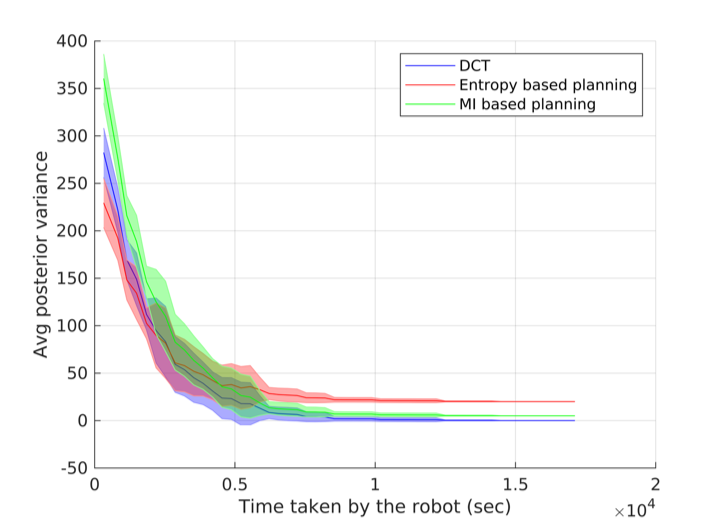}}
\caption{\DCT{} performs comparably with entropy-based and MI-based strategies. The shaded regions correspond to the standard deviation taken over ten trials.~\label{fig:compBaselines}}
\end{figure*}

\subsection{MSE and Variance}
We verify our hypothesis that MSE is equal to the posterior variance for GPs. A plot of the mean percent difference between the empirical MSE and the posterior variance is shown in Figure~\ref{fig:percent_diff}. \revone{The mean is computed over approximately 5600 test locations which are different from the measurement locations and placed on a grid.} As the number of trials increases, the mean difference between empirical MSE, which is essentially the MSE given enough number of trials, and the posterior variance decreases implying that the empirical MSE converges to the posterior variance asymptotically. \revone{In each trial, the measurement locations, test locations, and the hyperparameters are same, and therefore the variance estimates are same as well. However, the predicted value in each trial, and hence the prediction error, may be different since the actual measurement collected can be different in each trial due to the simulated noise. The effect of noise will decrease as one computes empirical estimate over a larger number of trials.}
\begin{figure}
    \centering
  \includegraphics[width=0.4\textwidth]{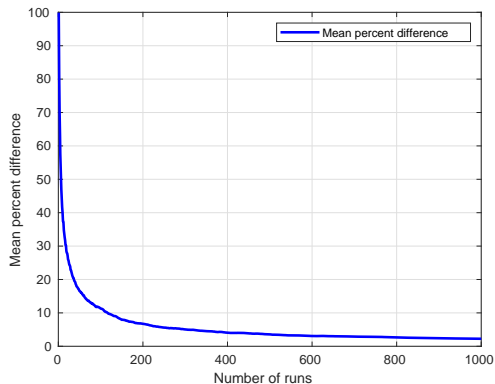}
  \caption{The mean percentage difference between the empirical MSE and the posterior variance. ~\label{fig:percent_diff}}.
\end{figure}

\section{Conclusion} \label{ch:conclusion}
In this paper, we study several problems: Placing the minimum number of stationary sensors to track a  spatial field, mapping a spatial field by a single as well as multiple robots while minimizing the time taken by the robots. For all the problems, \revone{we propose polynomial-time approximation algorithms to ensure that the mean square error in prediction the underlying spatial field is smaller than a pre-defined threshold at each point.} We also derive the lower bounds on the performance of any algorithm (including optimal) to solve respective problems are provided. We show that it is possible to learn a given spatial field accurately with high confidence without planning adaptively. Note that, if the kernel parameters are optimized online, then, one would require an adaptive strategy. 

The algorithms suggested in this paper perform comparatively with the baseline planners developed earlier. Our algorithms have theoretical bounds on their performance. \revone{The algorithms can also be generalized to 3D mapping, even though we illustrate using 2D examples. The disks in the 2D case will be replaced by spheres in 3D. The disk packing/covering problem becomes a sphere packing/covering. The tour will need to visit points in 3D, as opposed to 2D. The existing TSP algorithms already apply to the 3D case~\cite{861310}.} Our ongoing work is on developing competitive strategies for spatio-temporal learning and deriving similar guarantees for adaptive cases.

\section*{Acknowledgment}
This material is based upon work supported by the National Science Foundation under Grant number 1637915 and NIFA grant 2018-67007-28380.

\ifCLASSOPTIONcaptionsoff
  \newpage
\fi
\bibliographystyle{IEEEtran}
\bibliography{references}

\begin{thebibliography}{10}
\providecommand{\url}[1]{#1}
\csname url@samestyle\endcsname
\providecommand{\newblock}{\relax}
\providecommand{\bibinfo}[2]{#2}
\providecommand{\BIBentrySTDinterwordspacing}{\spaceskip=0pt\relax}
\providecommand{\BIBentryALTinterwordstretchfactor}{4}
\providecommand{\BIBentryALTinterwordspacing}{\spaceskip=\fontdimen2\font plus
\BIBentryALTinterwordstretchfactor\fontdimen3\font minus
  \fontdimen4\font\relax}
\providecommand{\BIBforeignlanguage}[2]{{%
\expandafter\ifx\csname l@#1\endcsname\relax
\typeout{** WARNING: IEEEtran.bst: No hyphenation pattern has been}%
\typeout{** loaded for the language `#1'. Using the pattern for}%
\typeout{** the default language instead.}%
\else
\language=\csname l@#1\endcsname
\fi
#2}}
\providecommand{\BIBdecl}{\relax}
\BIBdecl

\bibitem{aktar2009impact}
W.~Aktar, D.~Sengupta, and A.~Chowdhury, ``Impact of pesticides use in
  agriculture: their benefits and hazards,'' \emph{Interdisciplinary
  toxicology}, vol.~2, no.~1, pp. 1--12, 2009.

\bibitem{adzigbli2018assessing}
L.~Adzigbli and D.~Yuewen, ``Assessing the impact of oil spills on marine
  organisms,'' \emph{J Oceanogr Mar Res}, vol.~6, no. 179, p.~2, 2018.

\bibitem{Blachowski2014}
\BIBentryALTinterwordspacing
J.~Blachowski, ``Spatial analysis of the mining and transport of rock minerals
  (aggregates) in the context of regional development,'' \emph{Environmental
  Earth Sciences}, vol.~71, no.~3, pp. 1327--1338, Feb 2014. [Online].
  Available: \url{https://doi.org/10.1007/s12665-013-2539-0}
\BIBentrySTDinterwordspacing

\bibitem{1607983}
G.~{Werner-Allen}, K.~{Lorincz}, M.~{Ruiz}, O.~{Marcillo}, J.~{Johnson},
  J.~{Lees}, and M.~{Welsh}, ``Deploying a wireless sensor network on an active
  volcano,'' \emph{IEEE Internet Computing}, vol.~10, no.~2, pp. 18--25, March
  2006.

\bibitem{krause2008near}
A.~Krause, A.~Singh, and C.~Guestrin, ``Near-optimal sensor placements in
  gaussian processes: Theory, efficient algorithms and empirical studies,''
  \emph{Journal of Machine Learning Research}, vol.~9, no. Feb, pp. 235--284,
  2008.

\bibitem{singh2009efficient}
A.~Singh, A.~Krause, C.~Guestrin, and W.~J. Kaiser, ``Efficient informative
  sensing using multiple robots,'' \emph{Journal of Artificial Intelligence
  Research}, vol.~34, pp. 707--755, 2009.

\bibitem{ouyang2014multi}
R.~Ouyang, K.~H. Low, J.~Chen, and P.~Jaillet, ``Multi-robot active sensing of
  non-stationary gaussian process-based environmental phenomena,'' in
  \emph{Proceedings of the 2014 international conference on Autonomous agents
  and multi-agent systems}.\hskip 1em plus 0.5em minus 0.4em\relax
  International Foundation for Autonomous Agents and Multiagent Systems, 2014,
  pp. 573--580.

\bibitem{hollinger2014sampling}
G.~A. Hollinger and G.~S. Sukhatme, ``Sampling-based robotic information
  gathering algorithms,'' \emph{The International Journal of Robotics
  Research}, vol.~33, no.~9, pp. 1271--1287, 2014.

\bibitem{ling2016gaussian}
C.~K. Ling, K.~H. Low, and P.~Jaillet, ``Gaussian process planning with
  lipschitz continuous reward functions: Towards unifying bayesian
  optimization, active learning, and beyond.'' in \emph{AAAI}, 2016, pp.
  1860--1866.

\bibitem{tokekar2016sensor}
P.~Tokekar, J.~{Vander Hook}, D.~Mulla, and V.~Isler, ``Sensor planning for a
  symbiotic {UAV} and {UGV} system for precision agriculture,'' \emph{{IEEE}
  Transactions on Robotics}, vol.~32, no.~6, pp. 1498--1511, 2016.

\bibitem{rasmussen2006gaussian}
C.~E. Rasmussen and C.~K.~I. Williams, \emph{Gaussian Processes for Machine
  Learning}.\hskip 1em plus 0.5em minus 0.4em\relax The MIT Press, 2006.

\bibitem{schulz2018tutorial}
E.~Schulz, M.~Speekenbrink, and A.~Krause, ``A tutorial on gaussian process
  regression: Modelling, exploring, and exploiting functions,'' \emph{Journal
  of Mathematical Psychology}, vol.~85, pp. 1--16, 2018.

\bibitem{krause2008efficient}
A.~Krause, J.~Leskovec, C.~Guestrin, J.~VanBriesen, and C.~Faloutsos,
  ``Efficient sensor placement optimization for securing large water
  distribution networks,'' \emph{Journal of Water Resources Planning and
  Management}, vol. 134, no.~6, pp. 516--526, 2008.

\bibitem{waagberg2016prediction}
J.~W{\aa}gberg, D.~Zachariah, T.~B. Sch{\"o}n, and P.~Stoica, ``Prediction
  performance after learning in gaussian process regression,'' \emph{arXiv
  preprint arXiv:1606.03865}, 2016.

\bibitem{yfantis1987efficiency}
E.~A. Yfantis, G.~T. Flatman, and J.~V. Behar, ``Efficiency of kriging
  estimation for square, triangular, and hexagonal grids,'' \emph{Mathematical
  Geology}, vol.~19, no.~3, pp. 183--205, 1987.

\bibitem{suryan2018sensor}
V.~Suryan and P.~Tokekar, ``Learning a spatial field with gaussian process
  regression in minimum time,'' in \emph{The 13th International Workshop on the
  Algorithmic Foundations of Robotics}, 2018.

\bibitem{bian2006utility}
F.~Bian, D.~Kempe, and R.~Govindan, ``Utility based sensor selection,'' in
  \emph{Proceedings of the 5th international conference on Information
  processing in sensor networks}.\hskip 1em plus 0.5em minus 0.4em\relax ACM,
  2006, pp. 11--18.

\bibitem{welch1982branch}
W.~J. Welch, ``Branch-and-bound search for experimental designs based on d
  optimality and other criteria,'' \emph{Technometrics}, vol.~24, no.~1, pp.
  41--48, 1982.

\bibitem{kirkpatrick1983optimization}
S.~Kirkpatrick, C.~D. Gelatt, and M.~P. Vecchi, ``Optimization by simulated
  annealing,'' \emph{science}, vol. 220, no. 4598, pp. 671--680, 1983.

\bibitem{hochbaum1985approximation}
D.~S. Hochbaum and W.~Maass, ``Approximation schemes for covering and packing
  problems in image processing and vlsi,'' \emph{Journal of the ACM (JACM)},
  vol.~32, no.~1, pp. 130--136, 1985.

\bibitem{Gonzalez-Banos:2001:RAA:378583.378674}
H.~Gonz\'{a}lez-Banos, ``A randomized art-gallery algorithm for sensor
  placement,'' in \emph{Proceedings of the Seventeenth Annual Symposium on
  Computational Geometry}, ser. SCG '01.\hskip 1em plus 0.5em minus 0.4em\relax
  New York, NY, USA: ACM, 2001, pp. 232--240.

\bibitem{caselton1984optimal}
W.~F. Caselton and J.~V. Zidek, ``Optimal monitoring network designs,''
  \emph{Statistics \& Probability Letters}, vol.~2, no.~4, pp. 223--227, 1984.

\bibitem{cressie1992statistics}
N.~Cressie, ``Statistics for spatial data,'' \emph{Terra Nova}, vol.~4, no.~5,
  pp. 613--617, 1992.

\bibitem{shewry1987maximum}
M.~C. Shewry and H.~P. Wynn, ``Maximum entropy sampling,'' \emph{Journal of
  applied statistics}, vol.~14, no.~2, pp. 165--170, 1987.

\bibitem{10.2307/171694}
\BIBentryALTinterwordspacing
C.-W. Ko, J.~Lee, and M.~Queyranne, ``An exact algorithm for maximum entropy
  sampling,'' \emph{Operations Research}, vol.~43, no.~4, pp. 684--691, 1995.
  [Online]. Available: \url{http://www.jstor.org/stable/171694}
\BIBentrySTDinterwordspacing

\bibitem{ramakrishnan2005gaussian}
N.~Ramakrishnan, C.~Bailey-Kellogg, S.~Tadepalli, and V.~N. Pandey, ``Gaussian
  processes for active data mining of spatial aggregates,'' in
  \emph{Proceedings of the 2005 SIAM International Conference on Data
  Mining}.\hskip 1em plus 0.5em minus 0.4em\relax SIAM, 2005, pp. 427--438.

\bibitem{zimmerman2006optimal}
D.~L. Zimmerman, ``Optimal network design for spatial prediction, covariance
  parameter estimation, and empirical prediction,'' \emph{Environmetrics},
  vol.~17, no.~6, pp. 635--652, 2006.

\bibitem{Zhu2006}
\BIBentryALTinterwordspacing
Z.~Zhu and M.~L. Stein, ``Spatial sampling design for prediction with estimated
  parameters,'' \emph{Journal of Agricultural, Biological, and Environmental
  Statistics}, vol.~11, no.~1, p.~24, Mar 2006. [Online]. Available:
  \url{https://doi.org/10.1198/108571106X99751}
\BIBentrySTDinterwordspacing

\bibitem{nemhauser1978analysis}
G.~L. Nemhauser, L.~A. Wolsey, and M.~L. Fisher, ``An analysis of
  approximations for maximizing submodular set functions—i,''
  \emph{Mathematical Programming}, vol.~14, no.~1, pp. 265--294, 1978.

\bibitem{atkinson2014optimal}
A.~C. Atkinson, ``Optimal design,'' \emph{Wiley StatsRef: Statistics Reference
  Online}, pp. 1--17, 2014.

\bibitem{6466575}
L.~{Van Nguyen}, S.~{Kodagoda}, R.~{Ranasinghe}, and G.~{Dissanayake},
  ``Simulated annealing based approach for near-optimal sensor selection in
  gaussian processes,'' in \emph{2012 International Conference on Control,
  Automation and Information Sciences (ICCAIS)}, Nov 2012, pp. 142--147.

\bibitem{jawaid2015informative}
S.~T. Jawaid and S.~L. Smith, ``Informative path planning as a maximum
  traveling salesman problem with submodular rewards,'' \emph{Discrete Applied
  Mathematics}, vol. 186, pp. 112--127, 2015.

\bibitem{krause2007nonmyopic}
A.~Krause and C.~Guestrin, ``Nonmyopic active learning of gaussian processes:
  an exploration-exploitation approach,'' in \emph{Proceedings of the 24th
  international conference on Machine learning}.\hskip 1em plus 0.5em minus
  0.4em\relax ACM, 2007, pp. 449--456.

\bibitem{low2012decentralized}
K.~H. Low, J.~Chen, J.~M. Dolan, S.~Chien, and D.~R. Thompson, ``Decentralized
  active robotic exploration and mapping for probabilistic field classification
  in environmental sensing,'' in \emph{Proceedings of the 11th International
  Conference on Autonomous Agents and Multiagent Systems-Volume 1}.\hskip 1em
  plus 0.5em minus 0.4em\relax International Foundation for Autonomous Agents
  and Multiagent Systems, 2012, pp. 105--112.

\bibitem{galland2004synthetic}
F.~Galland, P.~R{\'e}fr{\'e}gier, and O.~Germain, ``Synthetic aperture radar
  oil spill segmentation by stochastic complexity minimization,'' \emph{IEEE
  Geoscience and Remote Sensing Letters}, vol.~1, no.~4, pp. 295--299, 2004.

\bibitem{singh2006active}
A.~Singh, R.~Nowak, and P.~Ramanathan, ``Active learning for adaptive mobile
  sensing networks,'' in \emph{Proceedings of the 5th international conference
  on Information processing in sensor networks}.\hskip 1em plus 0.5em minus
  0.4em\relax ACM, 2006, pp. 60--68.

\bibitem{dantu2007detecting}
K.~Dantu and G.~S. Sukhatme, ``Detecting and tracking level sets of scalar
  fields using a robotic sensor network,'' in \emph{Proceedings 2007 IEEE
  International Conference on Robotics and Automation}.\hskip 1em plus 0.5em
  minus 0.4em\relax IEEE, 2007, pp. 3665--3672.

\bibitem{srinivasan2008ace}
S.~Srinivasan, K.~Ramamritham, and P.~Kulkarni, ``Ace in the hole: Adaptive
  contour estimation using collaborating mobile sensors,'' in \emph{Proceedings
  of the 7th international conference on Information processing in sensor
  networks}.\hskip 1em plus 0.5em minus 0.4em\relax IEEE Computer Society,
  2008, pp. 147--158.

\bibitem{bryan2006active}
B.~Bryan, R.~C. Nichol, C.~R. Genovese, J.~Schneider, C.~J. Miller, and
  L.~Wasserman, ``Active learning for identifying function threshold
  boundaries,'' in \emph{Advances in neural information processing systems},
  2006, pp. 163--170.

\bibitem{bryan2008actively}
B.~Bryan and J.~Schneider, ``Actively learning level-sets of composite
  functions,'' in \emph{Proceedings of the 25th international conference on
  Machine learning}.\hskip 1em plus 0.5em minus 0.4em\relax ACM, 2008, pp.
  80--87.

\bibitem{gotovos2013active}
A.~Gotovos, N.~Casati, G.~Hitz, and A.~Krause, ``Active learning for level set
  estimation,'' in \emph{Twenty-Third International Joint Conference on
  Artificial Intelligence}, 2013.

\bibitem{861310}
M.~{Wallat}, T.~{Graepel}, and K.~{Obermayer}, ``Gaussian process regression:
  active data selection and test point rejection,'' in \emph{Proceedings of the
  IEEE-INNS-ENNS International Joint Conference on Neural Networks. IJCNN 2000.
  Neural Computing: New Challenges and Perspectives for the New Millennium},
  vol.~3, July 2000, pp. 241--246 vol.3.

\bibitem{krause2011submodularity}
A.~Krause and C.~Guestrin, ``Submodularity and its applications in optimized
  information gathering,'' \emph{ACM Transactions on Intelligent Systems and
  Technology (TIST)}, vol.~2, no.~4, p.~32, 2011.

\bibitem{chekuri2005recursive}
C.~Chekuri and M.~Pal, ``A recursive greedy algorithm for walks in directed
  graphs,'' in \emph{46th Annual IEEE Symposium on Foundations of Computer
  Science (FOCS'05)}.\hskip 1em plus 0.5em minus 0.4em\relax IEEE, 2005, pp.
  245--253.

\bibitem{binney2013optimizing}
J.~Binney, A.~Krause, and G.~S. Sukhatme, ``Optimizing waypoints for monitoring
  spatiotemporal phenomena,'' \emph{The International Journal of Robotics
  Research}, vol.~32, no.~8, pp. 873--888, 2013.

\bibitem{das2008algorithms}
A.~Das and D.~Kempe, ``Algorithms for subset selection in linear regression,''
  in \emph{Proceedings of the fortieth annual ACM symposium on Theory of
  computing}.\hskip 1em plus 0.5em minus 0.4em\relax ACM, 2008, pp. 45--54.

\bibitem{zhang2007adaptive}
B.~Zhang and G.~S. Sukhatme, ``Adaptive sampling for estimating a scalar field
  using a robotic boat and a sensor network,'' in \emph{Robotics and
  Automation, 2007 IEEE International Conference on}.\hskip 1em plus 0.5em
  minus 0.4em\relax IEEE, 2007, pp. 3673--3680.

\bibitem{leonard2007collective}
N.~E. Leonard, D.~A. Paley, F.~Lekien, R.~Sepulchre, D.~M. Fratantoni, and
  R.~E. Davis, ``Collective motion, sensor networks, and ocean sampling,''
  \emph{Proceedings of the IEEE}, vol.~95, no.~1, pp. 48--74, 2007.

\bibitem{popa2006ekf}
D.~O. Popa, M.~F. Mysorewala, and F.~L. Lewis, ``Ekf-based adaptive sampling
  with mobile robotic sensor nodes,'' in \emph{2006 IEEE/RSJ International
  Conference on Intelligent Robots and Systems}.\hskip 1em plus 0.5em minus
  0.4em\relax IEEE, 2006, pp. 2451--2456.

\bibitem{sim2005global}
R.~Sim and N.~Roy, ``Global a-optimal robot exploration in slam,'' in
  \emph{Proceedings of the 2005 IEEE international conference on robotics and
  automation}.\hskip 1em plus 0.5em minus 0.4em\relax IEEE, 2005, pp. 661--666.

\bibitem{rahimi2004adaptive}
M.~Rahimi, R.~Pon, W.~J. Kaiser, G.~S. Sukhatme, D.~Estrin, and M.~Srivastava,
  ``Adaptive sampling for environmental robotics,'' in \emph{IEEE International
  Conference on Robotics and Automation, 2004. Proceedings. ICRA'04. 2004},
  vol.~4.\hskip 1em plus 0.5em minus 0.4em\relax IEEE, 2004, pp. 3537--3544.

\bibitem{low2007adaptive}
K.~H. Low, G.~J. Gordon, J.~M. Dolan, and P.~Khosla, ``Adaptive sampling for
  multi-robot wide-area exploration,'' in \emph{Proceedings 2007 IEEE
  International Conference on Robotics and Automation}.\hskip 1em plus 0.5em
  minus 0.4em\relax IEEE, 2007, pp. 755--760.

\bibitem{tuchscherer1993thompson}
A.~Tuchscherer, ``Thompson, steven k.: Sampling. john wiley \& sons, inc., new
  york--chichester--brisbane--toronto--singapore 1992, 343pp.,{\pounds} 41,
  95,'' \emph{Biometrical Journal}, vol.~35, no.~7, pp. 848--848, 1993.

\bibitem{low2008adaptive}
K.~H. Low, J.~M. Dolan, and P.~Khosla, ``Adaptive multi-robot wide-area
  exploration and mapping,'' in \emph{Proceedings of the 7th International
  Joint Conference on Autonomous agents and Multiagent systems-Volume 1}.\hskip
  1em plus 0.5em minus 0.4em\relax International Foundation for Autonomous
  Agents and Multiagent Systems, 2008, pp. 23--30.

\bibitem{low2009information}
------, ``Information-theoretic approach to efficient adaptive path planning
  for mobile robotic environmental sensing,'' in \emph{Nineteenth International
  Conference on Automated Planning and Scheduling}, 2009.

\bibitem{garnett2012bayesian}
R.~Garnett, Y.~Krishnamurthy, X.~Xiong, J.~Schneider, and R.~Mann, ``Bayesian
  optimal active search and surveying,'' \emph{arXiv preprint arXiv:1206.6406},
  2012.

\bibitem{yilmaz2008path}
N.~K. Yilmaz, C.~Evangelinos, P.~F. Lermusiaux, and N.~M. Patrikalakis, ``Path
  planning of autonomous underwater vehicles for adaptive sampling using mixed
  integer linear programming,'' \emph{IEEE Journal of Oceanic Engineering},
  vol.~33, no.~4, pp. 522--537, 2008.

\bibitem{cao2013multi}
N.~Cao, K.~H. Low, and J.~M. Dolan, ``Multi-robot informative path planning for
  active sensing of environmental phenomena: A tale of two algorithms,'' in
  \emph{Proceedings of the 2013 international conference on Autonomous agents
  and multi-agent systems}.\hskip 1em plus 0.5em minus 0.4em\relax
  International Foundation for Autonomous Agents and Multiagent Systems, 2013,
  pp. 7--14.

\bibitem{inproceedingsKobilarov}
Y.~Teck~Tan, A.~Kunapareddy, and M.~Kobilarov, ``Gaussian process adaptive
  sampling using the cross-entropy method for environmental sensing and
  monitoring,'' in \emph{International Conference on Robotics and Automation
  2018}, 05 2018, pp. 6220--6227.

\bibitem{srinivas2009gaussian}
N.~Srinivas, A.~Krause, S.~M. Kakade, and M.~W. Seeger, ``Information-theoretic
  regret bounds for gaussian process optimization in the bandit setting,''
  \emph{IEEE Transactions on Information Theory}, vol.~58, no.~5, pp.
  3250--3265, May 2012.

\bibitem{snelson2006sparse}
E.~Snelson and Z.~Ghahramani, ``Sparse gaussian processes using
  pseudo-inputs,'' in \emph{Advances in neural information processing systems},
  2006, pp. 1257--1264.

\bibitem{csato2002sparse}
L.~Csat{\'o} and M.~Opper, ``Sparse on-line gaussian processes,'' \emph{Neural
  computation}, vol.~14, no.~3, pp. 641--668, 2002.

\bibitem{Seeger:161318}
\BIBentryALTinterwordspacing
M.~Seeger, C.~Williams, and N.~Lawrence, ``Fast forward selection to speed up
  sparse gaussian process regression,'' \emph{Artificial Intelligence and
  Statistics 9}, 2003. [Online]. Available:
  \url{http://infoscience.epfl.ch/record/161318}
\BIBentrySTDinterwordspacing

\bibitem{Hoang:2015:UFA:3045118.3045180}
\BIBentryALTinterwordspacing
T.~N. Hoang, Q.~M. Hoang, and K.~H. Low, ``A unifying framework of anytime
  sparse gaussian process regression models with stochastic variational
  inference for big data,'' in \emph{Proceedings of the 32Nd International
  Conference on International Conference on Machine Learning - Volume 37}, ser.
  ICML'15.\hskip 1em plus 0.5em minus 0.4em\relax JMLR.org, 2015, pp. 569--578.
  [Online]. Available: \url{http://dl.acm.org/citation.cfm?id=3045118.3045180}
\BIBentrySTDinterwordspacing

\bibitem{mishra2018online}
R.~Mishra, M.~Chitre, and S.~Swarup, ``Online informative path planning using
  sparse gaussian processes,'' in \emph{2018 OCEANS-MTS/IEEE Kobe Techno-Oceans
  (OTO)}.\hskip 1em plus 0.5em minus 0.4em\relax IEEE, 2018, pp. 1--5.

\bibitem{luo2016distributed}
W.~Luo, S.~S. Khatib, S.~Nagavalli, N.~Chakraborty, and K.~Sycara,
  ``Distributed knowledge leader selection for multi-robot environmental
  sampling under bandwidth constraints,'' in \emph{2016 IEEE/RSJ International
  Conference on Intelligent Robots and Systems (IROS)}.\hskip 1em plus 0.5em
  minus 0.4em\relax IEEE, 2016, pp. 5751--5757.

\bibitem{cortes2004coverage}
J.~Cortes, S.~Martinez, T.~Karatas, and F.~Bullo, ``Coverage control for mobile
  sensing networks,'' \emph{IEEE Transactions on robotics and Automation},
  vol.~20, no.~2, pp. 243--255, 2004.

\bibitem{luo-and-sycara-2018-107469}
W.~Luo and K.~Sycara, ``Adaptive sampling and online learning in multi-robot
  sensor coverage with mixture of gaussian processes,'' in \emph{IEEE
  International Conference on Robotics and Automation (ICRA)}, May 2018.

\bibitem{kemna2017multi}
S.~Kemna, J.~G. Rogers, C.~Nieto-Granda, S.~Young, and G.~S. Sukhatme,
  ``Multi-robot coordination through dynamic voronoi partitioning for
  informative adaptive sampling in communication-constrained environments,'' in
  \emph{2017 IEEE International Conference on Robotics and Automation
  (ICRA)}.\hskip 1em plus 0.5em minus 0.4em\relax IEEE, 2017, pp. 2124--2130.

\bibitem{zlot2002multi}
R.~Zlot, A.~Stentz, M.~B. Dias, and S.~Thayer, ``Multi-robot exploration
  controlled by a market economy,'' in \emph{Proceedings 2002 IEEE
  International Conference on Robotics and Automation (Cat. No. 02CH37292)},
  vol.~3.\hskip 1em plus 0.5em minus 0.4em\relax IEEE, 2002, pp. 3016--3023.

\bibitem{simmons2000coordination}
\BIBentryALTinterwordspacing
R.~G. Simmons, D.~Apfelbaum, W.~Burgard, D.~Fox, M.~Moors, S.~Thrun, and
  H.~L.~S. Younes, ``Coordination for multi-robot exploration and mapping,'' in
  \emph{Proceedings of the Seventeenth National Conference on Artificial
  Intelligence and Twelfth Conference on Innovative Applications of Artificial
  Intelligence}.\hskip 1em plus 0.5em minus 0.4em\relax AAAI Press, 2000, pp.
  852--858. [Online]. Available:
  \url{http://dl.acm.org/citation.cfm?id=647288.723404}
\BIBentrySTDinterwordspacing

\bibitem{sheng2004multi}
W.~Sheng, Q.~Yang, S.~Ci, and N.~Xi, ``Multi-robot area exploration with
  limited-range communications,'' in \emph{2004 IEEE/RSJ International
  Conference on Intelligent Robots and Systems (IROS)(IEEE Cat. No.
  04CH37566)}, vol.~2.\hskip 1em plus 0.5em minus 0.4em\relax IEEE, 2004, pp.
  1414--1419.

\bibitem{soltero2012generating}
D.~E. Soltero, M.~Schwager, and D.~Rus, ``Generating informative paths for
  persistent sensing in unknown environments,'' in \emph{2012 IEEE/RSJ
  International Conference on Intelligent Robots and Systems}.\hskip 1em plus
  0.5em minus 0.4em\relax IEEE, 2012, pp. 2172--2179.

\bibitem{marino2015decentralized}
A.~Marino, G.~Antonelli, A.~P. Aguiar, A.~Pascoal, and S.~Chiaverini, ``A
  decentralized strategy for multirobot sampling/patrolling: Theory and
  experiments,'' \emph{IEEE Transactions on Control Systems Technology},
  vol.~23, no.~1, pp. 313--322, 2015.

\bibitem{meliou2007nonmyopic}
\BIBentryALTinterwordspacing
A.~Meliou, A.~Krause, C.~Guestrin, and J.~M. Hellerstein, ``Nonmyopic
  informative path planning in spatio-temporal models,'' in \emph{Proceedings
  of the 22nd National Conference on Artificial Intelligence - Volume 1}, ser.
  AAAI'07.\hskip 1em plus 0.5em minus 0.4em\relax AAAI Press, 2007, pp.
  602--607. [Online]. Available:
  \url{http://dl.acm.org/citation.cfm?id=1619645.1619742}
\BIBentrySTDinterwordspacing

\bibitem{hollinger2012uncertainty}
G.~A. Hollinger, B.~Englot, F.~Hover, U.~Mitra, and G.~S. Sukhatme,
  ``Uncertainty-driven view planning for underwater inspection,'' in \emph{2012
  IEEE International Conference on Robotics and Automation}.\hskip 1em plus
  0.5em minus 0.4em\relax IEEE, 2012, pp. 4884--4891.

\bibitem{o2012gaussian}
S.~T. O’Callaghan and F.~T. Ramos, ``Gaussian process occupancy maps,''
  \emph{The International Journal of Robotics Research}, vol.~31, no.~1, pp.
  42--62, 2012.

\bibitem{kennedy2001bayesian}
M.~C. Kennedy and A.~O'Hagan, ``Bayesian calibration of computer models,''
  \emph{Journal of the Royal Statistical Society: Series B (Statistical
  Methodology)}, vol.~63, no.~3, pp. 425--464, 2001.

\bibitem{oakley2004probabilistic}
J.~E. Oakley and A.~O'Hagan, ``Probabilistic sensitivity analysis of complex
  models: a bayesian approach,'' \emph{Journal of the Royal Statistical
  Society: Series B (Statistical Methodology)}, vol.~66, no.~3, pp. 751--769,
  2004.

\bibitem{hastie2005elements}
T.~Hastie, R.~Tibshirani, J.~Friedman, and J.~Franklin, ``The elements of
  statistical learning: data mining, inference and prediction,'' \emph{The
  Mathematical Intelligencer}, vol.~27, no.~2, pp. 83--85, 2005.

\bibitem{stein2012interpolation}
M.~L. Stein, \emph{Interpolation of spatial data: some theory for
  kriging}.\hskip 1em plus 0.5em minus 0.4em\relax Springer Science \& Business
  Media, 2012.

\bibitem{cover2012elements}
T.~M. Cover and J.~A. Thomas, \emph{Elements of information theory}.\hskip 1em
  plus 0.5em minus 0.4em\relax John Wiley \& Sons, 2012.

\bibitem{tekdas2010sensor}
O.~Tekdas and V.~Isler, ``Sensor placement for triangulation-based
  localization,'' \emph{IEEE transactions on Automation Science and
  Engineering}, vol.~7, no.~3, pp. 681--685, 2010.

\bibitem{646145}
S.~Arora, ``Nearly linear time approximation schemes for euclidean tsp and
  other geometric problems,'' in \emph{Proceedings 38th Annual Symposium on
  Foundations of Computer Science}, Oct 1997, pp. 554--563.

\bibitem{tekdas2012efficient}
O.~Tekdas, D.~Bhadauria, and V.~Isler, ``Efficient data collection from
  wireless nodes under the two-ring communication model,'' \emph{The
  International Journal of Robotics Research}, vol.~31, no.~6, pp. 774--784,
  2012.

\bibitem{4567906}
G.~N. Frederickson, M.~S. Hecht, and C.~E. Kim, ``Approximation algorithms for
  some routing problems,'' in \emph{17th Annual Symposium on Foundations of
  Computer Science (sfcs 1976)}, Oct 1976, pp. 216--227.

\bibitem{20023117691}
D.~J. Mulla, A.~C. Sekely, and M.~Beatty, ``Evaluation of remote sensing and
  targeted soil sampling for variable rate application of nitrogen.'' in
  \emph{Proceedings of the 5th International Conference on Precision
  Agriculture}, P.~C. Robert, R.~H. Rust, and W.~E. Larson, Eds.\hskip 1em plus
  0.5em minus 0.4em\relax Madison, USA: American Society of Agronomy, 2000, pp.
  1--15.

\bibitem{rasmussen2010gaussian}
C.~E. Rasmussen and H.~Nickisch, ``Gaussian processes for machine learning
  (gpml) toolbox,'' \emph{Journal of machine learning research}, vol.~11, no.
  Nov, pp. 3011--3015, 2010.

\end{thebibliography}

\begin{IEEEbiography}[{\includegraphics[width=1in, height=1.1in]{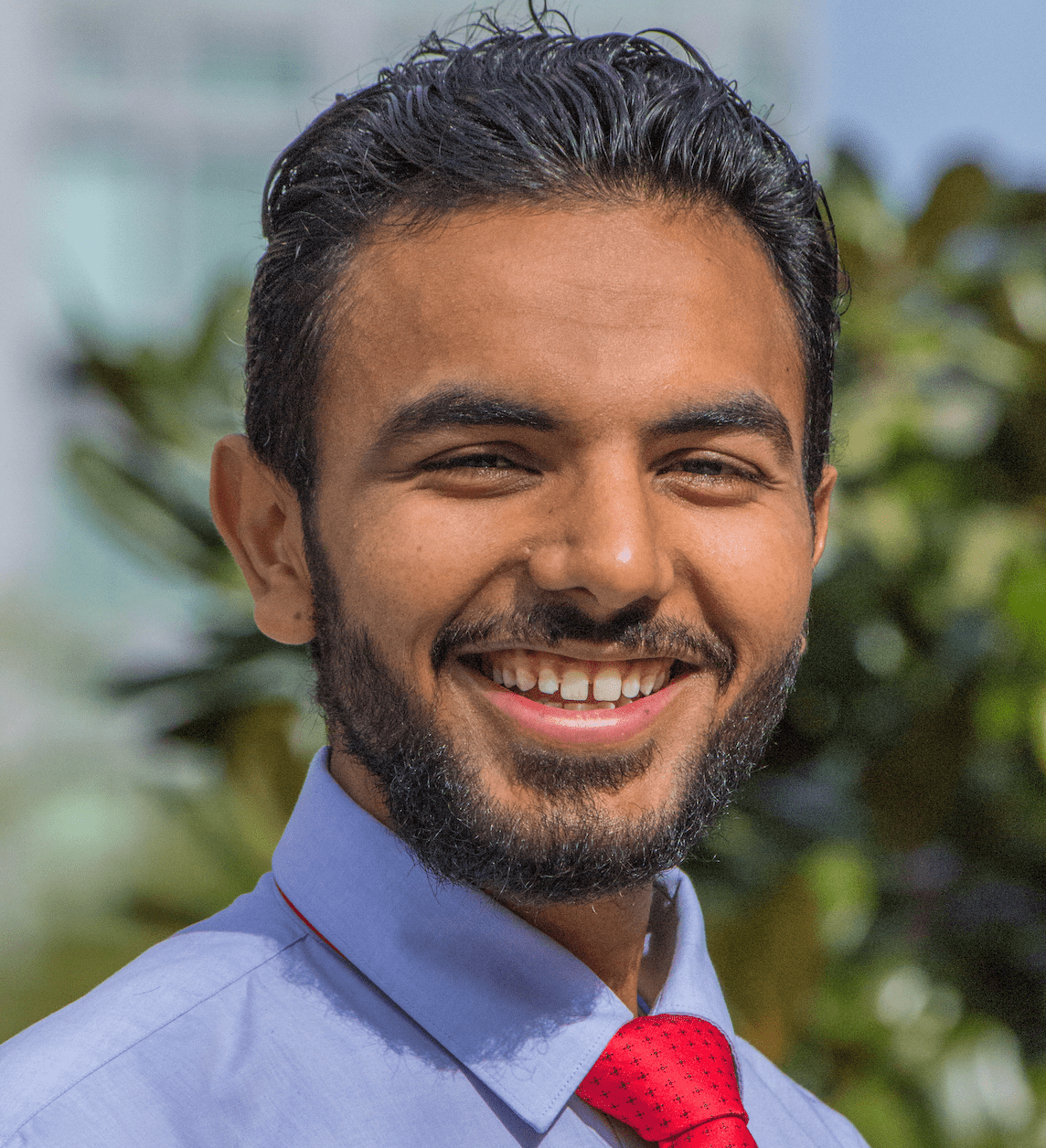}}]{Varun Suryan}
received his B.Tech. degree in mechanical engineering from Indian Institute of Technology Jodhpur, India, in 2016 and MS degree in computer engineering from Virginia Tech, Blacksburg, VA, USA, in 2019. He is pursuing his Ph.D. in computer science at University of Maryland, College Park, MD, USA. His research interests include algorithmic robotics, gaussian processes, and reinforcement learning.
\end{IEEEbiography}

\begin{IEEEbiography}[{\includegraphics[width=1in, height=1.1in]{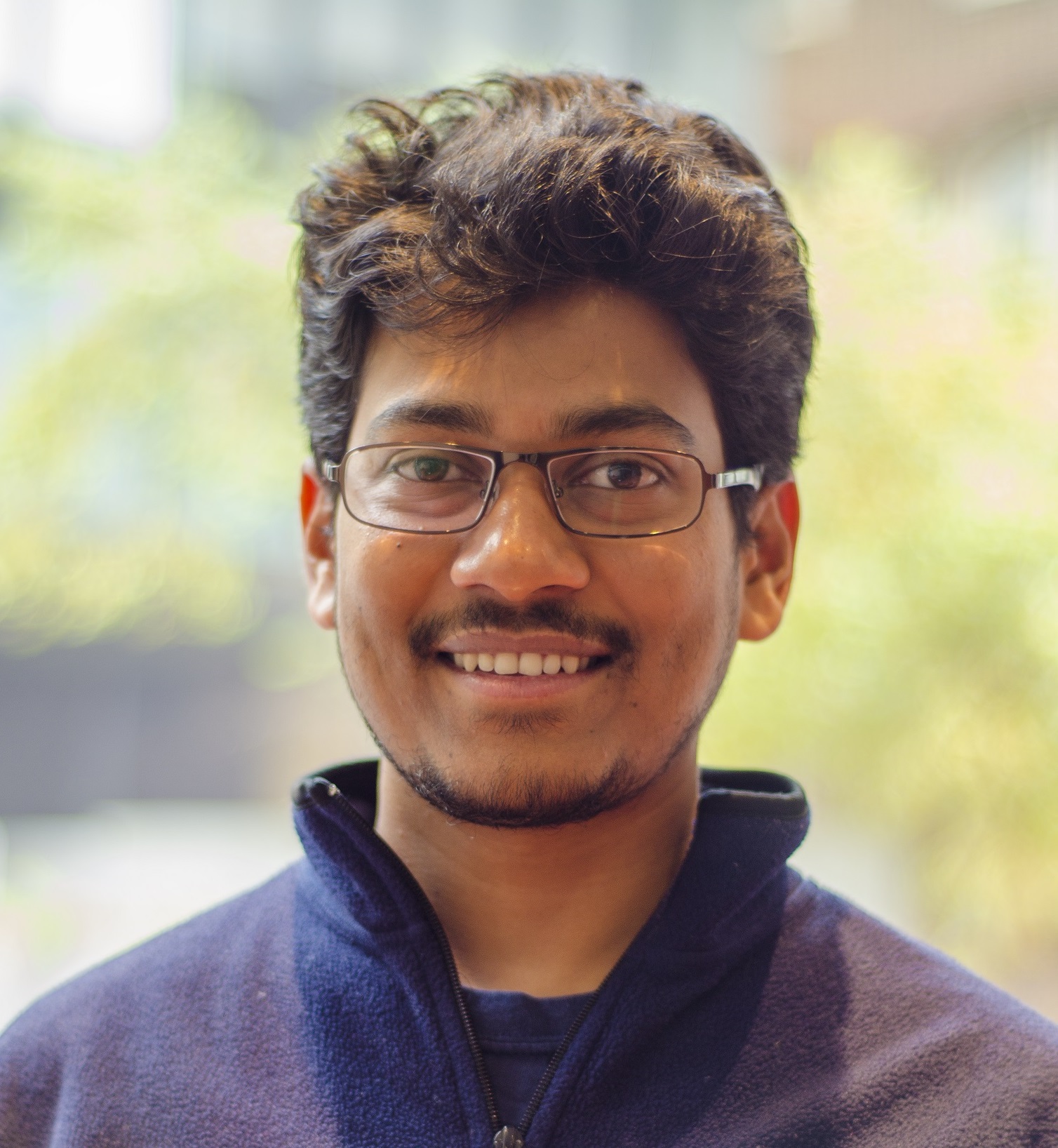}}]{Pratap Tokekar}
is an Assistant Professor in the Department of Computer Science at the University of Maryland. Previously, he was a Postdoctoral Researcher at the GRASP lab of University of Pennsylvania. Between 2015 and 2019, he was an Assistant Professor at the Department of Electrical and Computer Engineering at Virginia Tech. He obtained his Ph.D. in Computer Science from the University of Minnesota in 2014 and Bachelor of Technology degree in Electronics and Telecommunication from College of Engineering Pune, India in 2008. He is a recipient of the NSF CISE Research Initiation Initiative award and an Associate Editor for the IEEE Robotics and Automation Letters and Transactions of Automation Science and Engineering.
\end{IEEEbiography}

\begin{appendices}
		\section{} \label{app:appendix_one}
		\begin{proof}
		Consider two measurement locations $x_1, x_2$ and a test location $x$ such that $x_1$ is closer to $x$. The posterior variance at $x$ if a measurement was collected at $x_1$ can be computed as follows:
		\begin{align}
		    \hat{\sigma}^2_{x|x_1}&=k(x,x) - k(x, x_1)K(x_1, x_1)^{-1}k(x_1,x)\\
		    &=\sigma_0^2\left(1-\exp{\left(-\frac{||x-x_1||^2}{l^2}\right)}\right).\label{eq:nearVar}
		\end{align}
		Similarly, the posterior variance at $x$ if a measurement was collected at $x_2$, 
		\begin{align}
    \hat{\sigma}^2_{x|x_2} = \sigma_0^2\left(1-\exp{\left(-\frac{||x-x_2||^2}{l^2}\right)}\right).\label{eq:farVar}
		\end{align}
		From $||x-x_1||^2<||x-x_2||^2$ and $f(x)=-\exp{(-x)}$ being a monotonically increasing function, we have
		\begin{align}
-\exp{\left(-\frac{||x-x_1||^2}{l^2}\right)}<-\exp{\left(-\frac{||x-x_2||^2}{l^2}\right)}.
		\end{align}
		Using this to compare Equations~\ref{eq:nearVar} and~\ref{eq:farVar} one can easily see that $\hat{\sigma}^2_{x|x_1}<\hat{\sigma}^2_{x|x_2}$.
		\end{proof}
	\end{appendices}

\end{document}